\documentclass[hidelinks,onefignum,onetabnum]{siamart251216}
\usepackage{mathtools}
\usepackage{enumitem}
\usepackage{subcaption}
\usepackage{tikz}
\usetikzlibrary{arrows,matrix,fit,positioning}

\usepackage{lipsum}
\usepackage{amsfonts}
\usepackage{graphicx}
\usepackage{epstopdf}
\usepackage{algorithmic}
\ifpdf
  \DeclareGraphicsExtensions{.eps,.pdf,.png,.jpg}
\else
  \DeclareGraphicsExtensions{.eps}
\fi

\newsiamremark{remark}{Remark}
\newsiamremark{hypothesis}{Hypothesis}
\crefname{hypothesis}{Hypothesis}{Hypotheses}
\newsiamthm{claim}{Claim}
\newsiamremark{fact}{Fact}
\crefname{fact}{Fact}{Facts}

\headers{Structure-Preserving Neural Surrogates with Tractable UQ}{Zhang et al.}

\title{Structure-Preserving Neural Surrogates with Tractable Uncertainty Quantification \thanks{
\funding{This material is based upon work supported by the U.S. Department of Energy, Office of Science, Office of Advanced Scientific Computing Research, under award numbers DE-SC0024563 and DE-SC0023163.}
}}

\author{
Handi Zhang\thanks{Applied Mathematics and Computational Science, University of Pennsylvania, Philadelphia, PA, USA
  (\email{handi@sas.upenn.edu}).}
\and Adrienne M. Propp\thanks{Institute for Computational and Mathematical Engineering, Stanford University, Stanford, CA, USA 
  (\email{propp@stanford.edu}).}
\and Brooks Kinch\thanks{Mechanical Engineering and Applied Mechanics, University of Pennsylvania, Philadelphia, PA, USA
  (\email{kinch@seas.upenn.edu}).}
\and Houman Owhadi\thanks{Department of Computing and Mathematical Sciences, California Institute of Technology, Pasadena, CA, USA
  (\email{owhadi@caltech.edu}).}
\and Nathaniel Trask\thanks{Mechanical Engineering and Applied Mechanics, University of Pennsylvania, Philadelphia, PA, USA
  (\email{ntrask@seas.upenn.edu}).}
}

\usepackage{amsopn}

\begin{document}

\maketitle

\begin{abstract}
Recent advances in scientific machine learning provide a means of near-real-time solution to partial differential equations (PDEs), but lack the theoretical underpinnings of conventional simulators that support contemporary verification and validation. In this work, we construct data-driven reduced-order models that serve as structure-preserving, real-time surrogates. Remarkably, the exterior calculus that imposes physical conservation structure also exposes topological structure that we use to build a Gaussian process (GP) representation of uncertainty in state-flux relationships, ultimately yielding a Dirichlet-to-Neumann map for quantities of interest with closed-form expressions for posterior uncertainty. We specifically propose structure-preserving $H(\mathrm{div})$--$L^2$ subspaces of conventional Raviart--Thomas and $dgP_0$ elements prescribed by a lightweight transformer. Reduced-order dynamics consistent with this subspace are learned by posing a conservation law in which a GP describes the fluxes between volumes. This work hinges on a novel interface between mixed FEM spaces and GP regression; when training is posed as the optimal recovery problem (ORP), the resulting GP regression can be written as an optimization problem with equality constraints that impose a conservation structure, amenable to a fast Schur-complement training strategy. The trained model can then be solved in real time with closed-form estimators for boundary fluxes driven by prescribed Dirichlet data. The paper includes RKHS posterior error bounds for linear functionals to support uncertainty quantification, as well as numerical experiments demonstrating the accuracy of the posterior distribution as a surrogate for error estimation.
\end{abstract}



\section{Problem overview and relation to the literature}\label{sec:intro}
Neural operators and other machine learning surrogates are emerging as practical alternatives to classical simulation due to their computational efficiency and ability to generalize across problem instances. However, the black-box nature of these methods remains a major impediment to adoption in scientific and engineering applications and creates challenges for uncertainty quantification (UQ), where errors can propagate through the learned solution operator and further affect downstream predictions without clear traceability~\cite{abdar2021review,NAP29212}. Moreover, for systems governed by partial differential equations (PDEs), the learned surrogate models should also respect underlying physical constraints, such as conservation laws and initial and boundary conditions.

Motivated by this gap, we propose a surrogate framework for PDE-governed systems that supports fast simulation while providing tractable posterior uncertainty. Consider the following abstract problem, which we later elaborate in \Cref{sec:method2}. Let $u \in Q$ be a field on $\Omega \subset \mathbb{R}^d$ (naturally, $Q \subset L^2(\Omega)$), assumed to satisfy an unknown conservation law $\nabla \cdot F(u) = f$, where $F \in V$ is a given flux function (naturally, $V \subset H(\mathrm{div};\Omega)$). We assume access to data consisting of sampled state and flux fields $u$ and $F$. Additionally, we allow for parametric dependence on a conditioning variable $Z$, which may represent constitutive relationships, geometry, or other factors, yielding the training set

\[ \mathcal D_N =\left\{\left(z^{(n)},u^{(n)}_{\text{data}},F^{(n)}_{\text{data}}\right)\right\}_{n=1}^N. \]

Our goal is to identify a functional form for $F$ that is consistent with the available data, supports tractable posterior estimation, and generalizes to boundary data not seen during training. This learning problem must be posed in an appropriate discrete setting so that the resulting surrogate preserves the structure of the underlying PDE. We assume that the flux can be decomposed in the form
\begin{equation}
\mathbf F= -\epsilon \nabla u +\mathcal{N}[u],\qquad \nabla\cdot \mathbf F=f,
\label{eq:flux-ansatz}
\end{equation} 
consisting of a diffusion term with $\epsilon > 0$ for numerical stability and a nonlinear correction $\mathcal{N}$ that identifies a state-to-flux map from data. In previous work we demonstrated how neural operators can serve as $\mathcal{N}$~\cite{kinch2025structure}; here we instead treat $\mathcal{N}$ as a Gaussian process (GP). Classical GPs are most tractable in small-data, low-dimensional settings, and a primary contribution of this work is that the reduced graph representation developed below brings the otherwise infinite-dimensional state-to-flux learning problem into precisely such a regime. 

\begin{figure}[htbp]
    \centering\includegraphics[width=\linewidth]{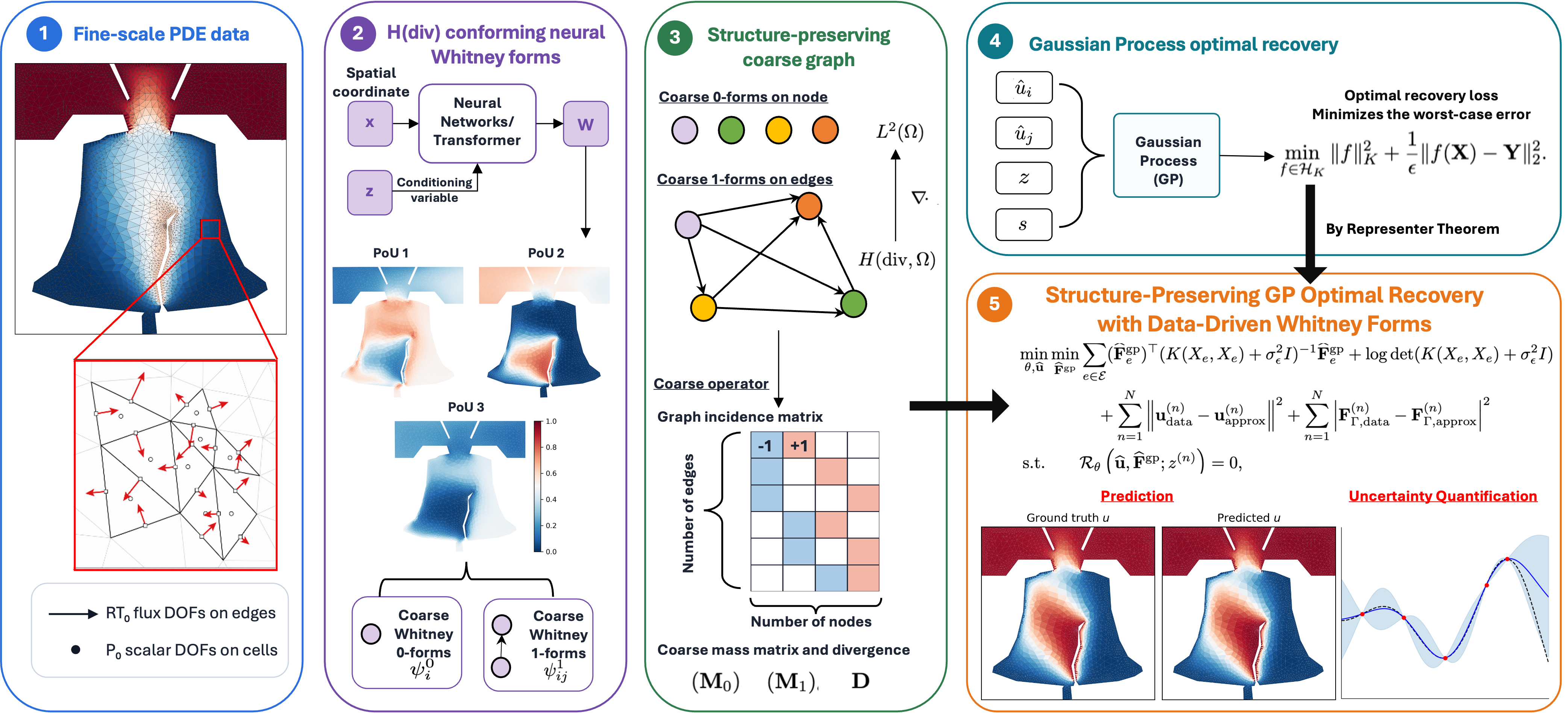}\caption{\textbf{Roadmap to structure-preserving surrogates with quantified uncertainty.} The transformer learns $H(\mathrm{div})$-conforming bases and constructs a conservative coarse graph adapted to the conditioning $Z$ (\textbf{Boxes}\textcircled{1}-\textcircled{3}). The graph edges carry GP models of the state-to-flux laws while conservation is imposed as an exact divergence constraint (\textbf{Box}\textcircled{4}) and the resulting reduced model serves as a Dirichlet-to-Neumann surrogate with posterior error estimates for boundary fluxes (\textbf{Box}\textcircled{5}).}\label{fig:roadmap}
\end{figure}

The proposed method separates the learning task into two parts: a data-driven model mapping the fine-scale space to a coarse-scale space (\textbf{P1}), and a stochastic model for flux correction that supports uncertainty quantification (\textbf{P2}). Because the full construction is technical, we first summarize the salient features of each.

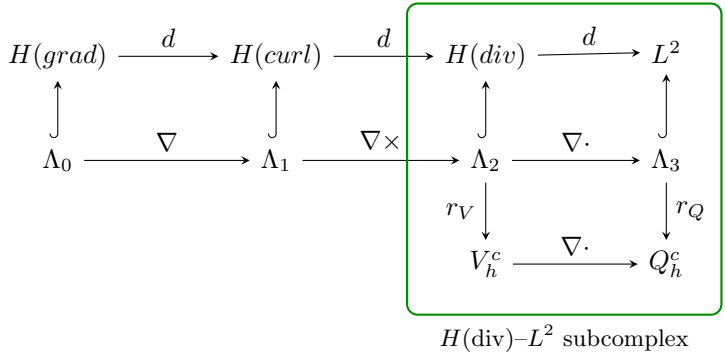
\begin{figure}[htbp]
\centering
\begin{tikzpicture}[
  >=stealth,
  arr/.style={->, shorten >=2pt, shorten <=2pt},
  inc/.style={right hook->, shorten >=2pt, shorten <=2pt},
]
  \matrix (m) [matrix of math nodes, row sep=2.6em, column sep=4.2em,
               nodes={inner sep=2pt}] {
    H(grad) & H(curl) & H(div) & L^2 \\
    \Lambda_0 & \Lambda_1 & \Lambda_2 & \Lambda_3 \\
        &     & V_h^c & Q_h^c \\
  };
  \draw[arr] (m-1-1) -- node[above]{$d$} (m-1-2);
  \draw[arr] (m-1-2) -- node[above]{$d$} (m-1-3);
  \draw[arr] (m-1-3) -- node[above]{$d$} (m-1-4);
  \draw[arr] (m-2-1) -- node[above]{$\nabla$} (m-2-2);
  \draw[arr] (m-2-2) -- node[above]{$\nabla\times$} (m-2-3);
  \draw[arr] (m-2-3) -- node[above]{$\nabla\cdot$} (m-2-4);
  \draw[arr] (m-3-3) -- node[above]{$\nabla\cdot$} (m-3-4);
  \draw[inc] (m-2-1) -- (m-1-1);
  \draw[inc] (m-2-2) -- (m-1-2);
  \draw[inc] (m-2-3) -- (m-1-3);
  \draw[inc] (m-2-4) -- (m-1-4);
  \draw[arr] (m-2-3) -- node[left]{$r_V$} (m-3-3);
  \draw[arr] (m-2-4) -- node[right]{$r_Q$} (m-3-4);
  \node[draw=green!55!black, rounded corners, thick,
        fit=(m-1-3)(m-1-4)(m-3-3)(m-3-4), inner sep=12pt] (box) {};
  \node[black, anchor=north] at (box.south) {\small $H(\mathrm{div})$--$L^2$ subcomplex};
\end{tikzpicture}
\caption{\textbf{Construction of $H(\mathrm{div})$-conforming reduced subspace.} We design a transformer that outputs a subspace of the de Rham complex appropriate for strongly imposing conservation laws. While previous work builds ``bottom-up" dual $\Lambda_0/\Lambda_1$ reduced subcomplexes \cite{actor2024data}, this work provides a ``top-down" primal subcomplex on $\Lambda_d/\Lambda_{d-1}$. A transformer prescribes the restriction map $r_Q$ conditioned on Z, and we provide a compatible construction of $r_V$ that defines coarsened $RT0-dgP_0$ subspaces which preserve surjectivity of the divergence.}
\label{fig:commuting}
\end{figure}

In \textbf{P1}, we work with low-order Whitney forms~\cite{arnold2010finite} that provide conforming finite element spaces encoding the topological and cohomological properties tied to conservation structure. Here the Raviart--Thomas and discontinuous-piecewise-constant ($\mathrm{RT}_0$/$dgP_0$) pair is naturally conforming, with $\mathrm{RT}_0\subset H(\mathrm{div};\Omega)$ and $dgP_0\subset L^2(\Omega)$, and has the surjectivity property $\nabla\cdot:\mathrm{RT}_0\rightarrow dgP_0$. Because this space interpolates cell-based scalar degrees of freedom and facet-based flux degrees of freedom, it provides a discrete divergence theorem, allowing the discrete div/grad matrices to be interpreted as adjacency matrices between cells and facets. This dual interpretation of div/grad is the crucial ingredient for Bayesian analysis on fields: in previous work~\cite{owhadi2022computational,propp2026discovery} we showed how to use optimal recovery to learn circuit models with tractable posteriors, and this linkage lets us extend the analysis to finite element fields.

The goal is to identify reduced-order spaces $Q_h^c(z;\theta)\subset Q^f_h$ and $V_h^c(z;\theta)\subset V^f_h$, where the superscripts $\cdot^f$ and $\cdot^c$ denote fine and coarse spaces and the subscript $\cdot_h$ denotes discretization. To preserve exterior calculus structure in the reduced space, we require a construction in which the restrictions $r_Q:Q^f_h \rightarrow Q^c_h$ and $r_V:V^f_h \rightarrow V^c_h$ commute with the divergence operator, so that $\nabla\cdot V_h^c\subseteq Q_h^c$. To achieve this, we first design a transformer that evaluates $r_Q$ by outputting a coarsening matrix $W$, so that $Q_h^c = \operatorname{span}\left\{\sum_a W_{ia} \chi_a\right\}_i$, where the $\chi_a$ are the indicator functions over the cells spanning $Q_h^f$. We then design $V_h^c$ by identifying coarsened degrees of freedom associated with coarse facets shared between coarse cells. This novel construction represents a ``top-down'' coarsening of the $(\Lambda_d/\Lambda_{d-1})$ subcomplex, in contrast to the ``bottom-up'' coarsening of $(\Lambda_0/\Lambda_{1})$ developed previously~\cite{actor2024data}. \textbf{P1} therefore produces a reduced $H(\mathrm{div})$-conforming finite element space that can be conditioned on $Z$.

In \textbf{P2}, we pose the discrete representation of the physics as an \textit{optimal recovery problem} \cite{owhadi2022computational}. We identify $F \in V_h^c$ and $u \in Q_h^c$ by their basis coefficients $\widehat{F}$ and $\widehat{u}$, where each coarse flux degree of freedom $\hat{F}_{ij}$ is associated with the coarse oriented boundary shared by the coarse cells $\hat{u}_{i}$ and $\hat{u}_{j}$ in the graph induced by \textbf{P1}. We model the nonlinearity edgewise by a Gaussian process, $\mathcal{N} = \mathcal{GP}(u_{ij})$, where $u_{ij}$ concatenates the states $\hat{u}_i,\hat{u}_j$ with the additional metric features needed to generalize the flux law across the mesh.

In traditional GP regression, the posterior distribution is derived from the joint normal distribution via a Schur complement. Its mean may equivalently be obtained by solving the optimal recovery problem: given a kernel $K$ with reproducing kernel Hilbert space (RKHS) $\mathcal{H}_K$ and $N$ noisy data pairs $(\mathbf{X},\mathbf{Y})=\{(\mathbf{x}_i,y_i)\}_{i=1}^N$ with noise variance $\sigma_\varepsilon^2$, the minimizer of the RKHS norm penalized by data misfit
\begin{equation}
  \hat{f}
  \;=\;
  \operatorname*{arg\,min}_{f\in\mathcal{H}_K}\;
  \|f\|_{\mathcal{H}_K}^2
  \;+\;
  \frac{1}{\sigma_\varepsilon^2}\sum_{i=1}^{N}\bigl(f(\mathbf{x}_i)-y_i\bigr)^2 ,
\end{equation}
recovers the conventional GP estimator
$\hat{f}(\cdot)=K(\cdot,\mathbf{X})\bigl(K(\mathbf{X},\mathbf{X})+\sigma_\varepsilon^2 I\bigr)^{-1}\mathbf{Y}$.

Recasting the conventional GP problem in this manner allows us to incorporate equality constraints directly. In the graph interpretation of \textbf{P1}, fluxes over $V_h^c$ are edge currents between nodes associated with states in $Q_h^c$. On each edge we cast an optimal recovery problem mapping state to flux, subject to the equality constraint that the conservation law holds at each node. The connectivity of the coarse bases imposes sparsity on this dense graph through a metric weighting, and after training we are left with a discrete boundary value problem whose boundary fluxes are encoded by GPs. Because this is a constrained optimization problem, its KKT conditions expose a saddle-point structure that we exploit to derive a fast optimizer (\Cref{subsec:bilevel}).

In concert, \textbf{P1} and \textbf{P2} amount to a \textit{graph discovery problem}: by concurrently training the transformer and solving the optimal recovery problem, we interpret the identification of reduced div/grad finite element spaces as the identification of a dense graph, sparsified by the connectivity of the bases, that encodes the flow of conserved quantities through the system (\Cref{fig:roadmap}). While circuit analogies are commonplace in engineering (e.g., compact models for semiconductor devices~\cite{pmlr-v107-aadithya20a,fan2023two}, hydraulic circuits in hydrodynamics~\cite{9385620,vacca2021hydraulic}, or thermal circuits for heat transfer~\cite{wang2017microscale}), they are typically built from simplified analytic solutions and empirical curve-fitting through a slow iterative process. In contrast, our approach can be carried out autonomously to obtain rapid, uncertainty-quantified surrogates that preserve a conservative input/output relationship. We demonstrate this by considering two representative examples. Advection-diffusion on a triangular mesh of Philadelphia's Liberty Bell serves to illustrate how real-time surrogates with quantified uncertainty can be constructed on complex geometries; in this setting, we condition on the direction of advection as an example of how to construct parametric models. We finally construct a digital twin of a semiconductor device (specifically, a $p-n$ diode). Training data can be constructed by TCAD simulations solving the device drift-diffusion equations~\cite{musson2022charon}, providing examples of internal device transport and how it defines the voltage-current relationship governing the device. This problem specifically shows how the UQ developed here identifies the range of inputs over which the learned surrogate can be trustworthy.

\subsection{Relation to literature}
The construction above brings together three goals usually pursued separately: fast reduced-order simulation, exact enforcement of physical structure, and uncertainty quantification. We review each in turn. Classical discretizations such as finite element methods yield accurate, verifiable predictions but at substantial cost, motivating a large body of scientific machine learning (SciML) that trades rigor for speed. Physics-informed neural networks (PINNs) embed the governing PDEs in the training loss~\cite{karniadakis2021physics, yu2022gradient}; neural operators such as DeepONet~\cite{deeponetNatureML, pideeponet} and Fourier neural operators~\cite{li2020fourier} learn maps between function spaces; and data-driven reduced-order models (ROMs) accelerate parametric simulation by working in a low-dimensional subspace~\cite{fresca2022pod, jung2025accelerating, kapteyn2022data}. These surrogates can be fast, but typically forfeit the structural guarantees and error control of the solvers they replace. A related line of structure-preserving model-reduction methods addresses part of this gap by seeking reduced spaces that retain conservation, stability, passivity, or compatibility properties of the full-order discretization \cite{benner2015survey,carlberg2018conservative,quarteroni2015reduced}. Our construction is close in spirit to this line of work, but learns the compatible reduced $H(\mathrm{div})-L^2$ complex from data rather than selecting it through a fixed projection, and couples it to a GP optimal-recovery formulation for posterior uncertainty.

Enforcing physical structure reliably is the central difficulty, particularly under limited data and complex geometry. Most physics-informed approaches impose conservation through soft penalties on the PDE residual or boundary conditions~\cite{chen2024physics,jiao2024solving}; these do not guarantee exact conservation, and the resulting violations can accumulate, amplify errors, and destabilize training on stiff or multiscale problems~\cite{bonfanti2024challenges,krishnapriyan2021characterizing}. Structure-preserving methods instead build the constraints into the function spaces themselves. The mixed finite element theory underlying Raviart-Thomas and discontinuous Galerkin pairs provides the classical foundation for such $H(\mathrm{div})$-conforming flux approximations and compatible $L^2$ scalar spaces; see, for example, \cite{boffi2013mixed}. Finite element exterior calculus (FEEC) provides a canonical framework for this approach, choosing spaces and operators that form a discrete de Rham complex and so inherit the topological identities underlying conservation~\cite{arnold2010finite, Arnold_Falk_Winther_2006}. Data-driven exterior calculus extends this to learned, graph-based operators while preserving exact-sequence compatibility~\cite{trask2022enforcing}. Our coarsened $\mathrm{RT}_0$/$dgP_0$ complex is designed to inherit these guarantees by construction. 

Finally, trustworthy surrogates require calibrated uncertainty estimates that can be propagated to downstream quantities of interest~\cite{abdar2021review, psaros2023uncertainty, xu2021accurate}. UQ for neural PDE surrogates is often pursued through ensembles, Bayesian neural networks, dropout, latent-variable models, or post-hoc calibration; these approaches can be effective but usually do not yield the closed-form functional posterior bounds available in GP/RKHS optimal recovery \cite{song2026structure,yang2019adversarial}. Gaussian processes provide a natural alternative because conditioning a GP prior on data yields a posterior mean and covariance in closed-form. The same estimator arises from the optimal recovery problem, which characterizes the minimum-norm RKHS interpolant of noisy data~\cite{chen2021solving, owhadi2022computational, propp2026discovery} and, as we exploit below, accommodates hard linear constraints. We build most directly on the graph-based optimal recovery of Dirichlet-to-Neumann maps in~\cite{propp2026discovery}, extending it from fixed graphs to the learned, conditioned finite element complexes produced by \textbf{P1}.

\subsection{Main contributions}
\Cref{fig:roadmap} summarizes the resulting framework with structure preservation and uncertainty quantification. Our specific contributions are:
\begin{itemize}
    \item a \emph{top-down} coarsening of the $(\Lambda_d/\Lambda_{d-1})$ subcomplex, in which a lightweight transformer emits a coarsening matrix $W$ that defines a reduced, $H(\mathrm{div})$-conforming $\mathrm{RT}_0$/$dgP_0$ pair conditioned on the problem parameters $Z$, in contrast to the bottom-up $(\Lambda_0/\Lambda_1)$ coarsening of~\cite{actor2024data};
    \item a structure-preserving optimal recovery formulation in which an edgewise Gaussian process models the state-to-flux law and conservation is imposed exactly through linear equality constraints, yielding a saddle-point KKT system with a fast Schur-complement solve;
    \item a joint ``graph discovery'' training procedure that learns the reduced complex and the flux Gaussian process concurrently, producing a sparse circuit model of the conserved dynamics that generalizes across geometries and boundary data;
    \item a closed-form RKHS posterior error bound for linear functionals of the flux, in particular the boundary fluxes defining the Dirichlet-to-Neumann map, validated on complex geometries and semiconductor device equations.
\end{itemize}
The remainder of the paper is organized as follows. 
\textbf{\Cref{sec:math_prelim} }reviews the necessary mathematical background for the proposed framework. \Cref{subsec:feec} reviews the finite element exterior calculus preliminaries for the $H(\mathrm{div})$--$L^2$ mixed setting and \Cref{subsec:hdiv_rt0} specifies the finite element space with conforming $\mathrm{RT}_0$/$dgP_0$ discretizations. \Cref{subsec:orp_gp} explains the optimal recovery problem for Gaussian processes. \textbf{\Cref{sec:method1}} develops the first component (\textbf{P1}) by introducing the data-driven PoU Whitney construction (\Cref{subsec:pou_whitney}) and the induced coarse operators (\Cref{subsec:coarse_operator})  for structure preservation. We show that under this construction, properties such as surjectivity and flux balance still hold in the reduced space (\Cref{subsec:surjective}). In addition, we provide a graph interpretation of the learned reduced spaces (\Cref{subsec:graph_cal}), which supports the subsequent optimal recovery formulation. \textbf{\Cref{sec:method2}} develops the second component (\textbf{P2}). We cast the GP optimal recovery formulation on the coarse spaces subject to equality constraints that guarantee the exact enforcement of conservation laws (\Cref{subsec:opt_formulation}). In addition to a fast Schur-complement solve for the KKT system, we propose a bilevel training procedure for efficient training with nested coarse variables and models (\Cref{subsec:bilevel}) and incorporate geometric embedding (\Cref{subsec:geom_embedding}). We also provide the theoretical analysis for the RKHS posterior error bound in \Cref{subsec:error_bound}. \textbf{\Cref{sec:results}} reports numerical experiments on representative examples, a complex-geometry advection-diffusion problem, and a semiconductor diode example to validate the proposed framework.  

\section{Mathematical preliminaries}\label{sec:math_prelim}
In this section we gather the necessary background in both finite element exterior calculus (FEEC) and the optimal recovery formalisms. For further background, please see \cite{arnold2010finite, Arnold_Falk_Winther_2006,micchelli1977survey} and \cite{trask2022enforcing}.

\subsection{Finite element exterior calculus (FEEC)}\label{subsec:feec}
Exterior calculus generalizes classical calculus to differential forms of higher degree on differentiable manifolds and provides a robust mathematical framework for analyzing PDEs on different geometries. Finite element exterior calculus extends the exterior calculus to discrete finite element meshes using chain complexes and preserves crucial topological and geometric structures underlying the governing PDEs~\cite{kinch2025structure, trask2022enforcing}. 

One advantage of FEEC is that it unifies the main differential operators, i.e., grad, div, and curl, with the de Rham complex. Let $\Lambda(\Omega)$ denote the exterior algebra with exterior product $\wedge$ and
$\Omega\subset\mathbb{R}^d$. We then have the de Rham complex:
\begin{equation}
    0\rightarrow \Lambda^{0}(\Omega)\xrightarrow{d^0}\Lambda^{1}(\Omega)\xrightarrow{d^1}\cdots\xrightarrow{d^{k-1}}\Lambda^{k}(\Omega)\rightarrow 0,
\end{equation}
where $d:\Lambda^{k-1}(\Omega)\rightarrow\Lambda^k(\Omega)$ is the exterior derivative operator. For $\Omega\subset\mathbb{R}^3$, the de Rham complex becomes:
\begin{equation}
    0\rightarrow C^\infty(\Omega)\xrightarrow{\text{grad}}[C^\infty(\Omega)]^3\xrightarrow{\text{curl}}[C^\infty(\Omega)]^3\xrightarrow{\text{div}}C^\infty(\Omega)\rightarrow 0.
\end{equation}
If we further consider Sobolev spaces with homogeneous Dirichlet boundary conditions, we can derive the FEEC specialization primal and dual cochain complex:
\begin{equation}
0
\longrightarrow
H_0(\mathrm{grad},\Omega)
\xrightleftharpoons[\,-\mathrm{grad}\,]{\ \mathrm{grad}\ }
H_0(\mathrm{curl},\Omega)
\xrightleftharpoons[\,\mathrm{curl}\,]{\ \mathrm{curl}\ }
H_0(\mathrm{div},\Omega)
\xrightleftharpoons[\,-\mathrm{div}\,]{\ \mathrm{div}\ }
L^2(\Omega)
\longrightarrow
0.
\end{equation}
In this work, instead of working on ``bottom-up" subcomplexes, we work on the terminal map \(H(\mathrm{div},\Omega)\xrightarrow{\nabla\cdot}L^2(\Omega)\) because this is a natural algebraic mechanism through which flux balance is imposed. This is illustrated in the commuting diagram in \Cref{fig:commuting}.
\subsection{$H(\mathrm{div})$ space and lowest-order Raviart--Thomas element}\label{subsec:hdiv_rt0}
Since one core objective of this work is to guarantee the flux conservation law, instead of working on the entire de Rham complex, we focus on its tail, i.e., $H(\mathrm{div},\Omega)\xrightarrow{\text{div}} L^2(\Omega)$. This segment of the complex supports flux conservation in a natural way through the mixed formulation of elliptic problems. To make this connection explicit, consider the general second-order elliptic equation in a bounded domain $\Omega \subset \mathbb{R}^d$:  
\begin{equation}\label{eq:elliptic_strong}
-\nabla\cdot\bigl(\mathbf{K}(x)\nabla u\bigr)=f
\qquad \text{in } \Omega,
\end{equation}
where $\mathbf{K}(x)\in\mathbb{R}^{d\times d}$ is a symmetric and uniformly positive definite diffusion tensor. The boundary $\partial\Omega$ is decomposed into disjoint parts
\(\partial\Omega=\Gamma_D \cup \Gamma_N\) on which Dirichlet and Neumann boundary conditions are defined, respectively. Then by introducing the flux variable $\mathbf{F}\coloneqq-\mathbf{K}(x)\nabla u$, the problem can be rewritten as the first-order system:
\begin{equation}\label{eq:mixed_strong_general}
\begin{cases}
  \mathbf{K}(x)^{-1}\mathbf{F}+\nabla u&=0,\\
  \nabla\cdot\mathbf{F}&=f,
\end{cases}\qquad \text{in } \Omega.
\end{equation}
In this setting, the natural space for the flux variable is:
\begin{equation}
    H(\mathrm{div}, \Omega)
\coloneqq
\left\{
\mathbf{v} \in [L^2(\Omega)]^d
\;\middle|\;
\nabla \cdot \mathbf{v} \in L^2(\Omega)
\right\}.
\end{equation} 
This space ensures that the divergence operator is well-defined in a weak sense and the normal components $\mathbf{v}\cdot\mathbf{n}$ are well-defined on element interfaces, thus naturally making it the appropriate space for modeling local flux conservation. The corresponding mixed variational formulation is to find $(\mathbf{F},u)\in H(\mathrm{div},\Omega)\times L^2(\Omega)$ such that
\begin{equation}\label{eq:mixed_weak_integral}
\begin{split}
\int_\Omega \mathbf{K}(x)^{-1}\mathbf{F}\cdot \mathbf{v}\,dx
-
\int_\Omega u\,(\nabla\cdot \mathbf{v})\,dx
&=-\int_{\Gamma_D} u_D\,(\mathbf{v}\cdot\mathbf{n})\,ds,
\quad\forall\,\mathbf{v}\in H(\mathrm{div},\Omega),\\    
\int_\Omega (\nabla\cdot\mathbf{F})\,q\,dx
&=\int_\Omega f\,q\,dx,\quad \forall\,q\in L^2(\Omega).
\end{split}
\end{equation}
A conforming discretization of this mixed formulation requires a finite element flux space in \(H(\mathrm{div},\Omega)\) and a compatible scalar space in \(L^2(\Omega)\). In this work, we consider the Raviart--Thomas (RT) elements for the flux variable and piecewise constants for the scalar variable. Let $\mathcal{T}_h$ be a simplicial mesh of $\Omega \subset \mathbb{R}^d$. On each element $K\in \mathcal{T}_h$, the Raviart--Thomas space of order $k \ge 0$  is defined by
\begin{equation}
\mathrm{RT}_k(K)\coloneqq P_k(K)^d + \mathbf{x}\,P_k(K),
\end{equation}
where ${P}_k(K)$ denotes the space of polynomials of degree at most $k$ on $K$. The global Raviart--Thomas space is defined as
\begin{equation}
\mathrm{RT}_k(\mathcal{T}_h)\coloneqq
\left\{
\mathbf{v} \in H(\mathrm{div},\Omega)
\;\middle|\;
\mathbf{v}|_K \in \mathrm{RT}_k(K), \; \forall K \in \mathcal{T}_h
\right\},
\end{equation}
which enforces the continuity of normal components across element interfaces.

For $k=0$, the \emph{lowest--order Raviart--Thomas space} on $K$ is defined as
\begin{equation}
    \text{RT}_{0}(K)
    \;=\;
    P_{0}(K)^{d}\;+\;x\,P_{0}(K).
\end{equation}
Equivalently, any $\mathbf{v} \in \mathrm{RT}_0(K)$ can be written as
$\mathbf{v}(x) = \mathbf{a} + p\,x$ with constant vector $\mathbf{a}\in\mathbb R^{d}$ and
constant scalar $p\in\mathbb R$. The degrees of freedom of this lowest-order RT$_0$ space are defined by the normal fluxes across element faces: 
\begin{equation}
    \int_e (\mathbf{v} \cdot \mathbf{n})\, q \, ds, \qquad \forall e \subset \partial K.
\end{equation}Thus the degree of freedom admits a direct interpretation in terms of cochains.  In particular, the RT space ensures that the discrete divergence operator is
compatible with the underlying differential complex, and the $(\mathrm{RT}_0, \mathrm{P}_0)$ can form a meaningful discretization of the $H(\mathrm{div})$--$L^2$ mixed formulation. Within the FEEC framework, this construction corresponds to the final part of the de Rham complex where the scalar field is associated with $d$-forms and the flux corresponds to $(d-1)$-forms. The gradient and divergence operators are unified through the exterior derivative and the codifferential. The resulting compatible structure provides the algebraic foundation for conservation and motivates the learnable coarse space construction below.

\subsection{Optimal recovery problem for Gaussian processes}\label{subsec:orp_gp}
Gaussian processes provide a probabilistic model for unknown functions. Formally, a GP is a 
collection of random variables such that any finite collection has a joint Gaussian distribution. Equivalently, a GP can be viewed as a distribution over functions: each draw from the process is a possible function, and for any finite set of input locations, the corresponding vector of function values is multivariate Gaussian. This makes GPs particularly useful in settings where the uncertainty quantification is important, since conditioning on observed data yields both a posterior mean prediction and a posterior covariance that quantifies uncertainty at new inputs. 

Given data $\mathcal{D}=(\mathbf{X},\mathbf{Y})$ with $X=[\mathbf{x}_1,\dots, \mathbf{x}_N]^\top$ and corresponding output $\mathbf{Y}=[y_1,\dots,y_N]^\top$ generated by an underlying function $y_i=f(\mathbf{x}_i)$, a GP prior on $f$ is fully specified by its mean function $m(x)$ and covariance kernel $K(x,x')$:
\begin{align}
    m(\mathbf{x})&=\mathbb{E}[f(\mathbf{x})]\\
    K(\mathbf{x},\mathbf{x}')&=\mathbb{E}[(f(\mathbf{x})-m(\mathbf{x}))(f(\mathbf{x}')-m(\mathbf{x}'))]\\
    f(\mathbf{x})&\sim \mathcal{GP}(m(\mathbf{x}),K(\mathbf{x},\mathbf{x}'))
\end{align}
In the noise-free scenario, for test inputs $\mathbf X_*=(\mathbf x_1^*,\ldots,\mathbf x_{N_*}^*)^T$, the joint distribution of the observed training values $\mathbf{Y}=f(\mathbf X)$ and the unobserved test values $\mathbf{f_*}=f(X_*)$ is\footnote{In \eqref{eq:gp_distribution}, we use a zero-mean GP prior for simplicity. This is a standard convention when the data have been centered or when the GP models a residual correction around a deterministic mean model. A nonzero mean can be incorporated by applying the same formulas to the centered quantities obtained after subtracting the mean. Under this zero-mean prior, the joint distribution of the observed training outputs and the unobserved test values is also zero mean.}
\begin{equation}\label{eq:gp_distribution}
    \begin{bmatrix}
\mathbf{Y}\\\mathbf{f}_* 
    \end{bmatrix}\sim\mathcal{N}\left(\mathbf{0},  \begin{bmatrix}
        K(\mathbf X,\mathbf X)&K(\mathbf X,\mathbf X_*)\\
        K(\mathbf X_*,\mathbf X)&K(\mathbf X_*,\mathbf X_*)
    \end{bmatrix}\right).
\end{equation}
\
Conditioning on this joint Gaussian distribution gives
\begin{equation}
  \mathbf{f}_* |\mathbf X_*, \mathbf X,\mathbf{Y}\sim\mathcal{N}(\mu_*,
\Sigma_*),\label{eq:gp_post}
\end{equation}
where 
\begin{align}
    \mu_*
&=K(\mathbf X_*,\mathbf X)K(\mathbf X,\mathbf X)^{-1}\mathbf{Y},\\
  \Sigma_*&=K(\mathbf X_*,\mathbf X_*)-K(\mathbf X_*,\mathbf X)K(\mathbf X,\mathbf X)^{-1}K(\mathbf X,\mathbf X_*).
\end{align}
In most applications, however, the data are noisy. Suppose $\mathbf{Y}=f(\mathbf X)+\varepsilon$, where $\varepsilon\sim\mathcal N(0,\sigma_\varepsilon^2 I)$ is independent identically distributed (i.i.d.) Gaussian noise. Then the joint distribution in \eqref{eq:gp_distribution}-\eqref{eq:gp_post} becomes 
\begin{align}\label{eq:gp_distribution_noise}
    \begin{bmatrix}
\mathbf{Y}\\\mathbf{f}_* 
    \end{bmatrix}&\sim\mathcal{N}\left(\mathbf{0},  \begin{bmatrix}
        K(\mathbf X,\mathbf X)+\sigma_\varepsilon^2I&K(\mathbf X,\mathbf X_*)\\
        K(\mathbf X_*,\mathbf X)&K(\mathbf X_*,\mathbf X_*)
    \end{bmatrix}\right),\\
  \mathbf{f}_* |\mathbf X_*, \mathbf X,\mathbf{Y}&\sim\mathcal{N}\Big(K(\mathbf X_*,\mathbf X)(K(\mathbf X,\mathbf X)+\sigma_\varepsilon^2I)^{-1}\mathbf{Y},\notag\\
  &\qquad \quad K(\mathbf X_*,\mathbf X_*)-K(\mathbf X_*,\mathbf X)(K(\mathbf X,\mathbf X)+\sigma_{\varepsilon}^2I)^{-1}K(\mathbf X,\mathbf X_*)\Big).\label{eq:gp_post_noise}
\end{align}
The posterior covariance provides a meaningful probabilistic measure of uncertainty at the test points, which is useful for uncertainty-aware scientific machine learning tasks. Indeed, prior work has established that incorporating GPs within SciML frameworks provides rigorous uncertainty quantification for learning complex PDE-governed systems~\cite{brunton2024promising, chen2021solving, harkonen2023gaussian, shi2025survey}.

Given limited noisy data and some prior knowledge of the function space, the optimal recovery problem aims to learn an unknown function by minimizing the worst-case error. Mathematically, given a normed space $F$ and an unknown function $f\in F$, the partial knowledge of $f$ is defined through point evaluations:
\begin{align}
    y_i=\ell_i(f), \quad i=1,\dots, m,
\end{align}
for some linear functionals $\ell_1,\dots,\ell_m\in F^*$, where $F^*$ is the dual space of $F$. The objective is to approximate $f$ by some estimate $\widehat{f}\in F$ such that the worst-case error is minimized:
\[
\mathcal{E}(\widehat{f})\coloneqq\inf_{\widehat{f}\in F}\sup_{\substack{f\in F\\ \ell_i(f)=y_i}}\frac{\|f-\widehat{f}\|}{\|f\|}.
\]
Let \( K: \mathcal{X}\times\mathcal{X} \rightarrow \mathbb{R} \) be a symmetric positive definite bivariate kernel, and \(\mathcal{H}_K \) denote its reproducing kernel Hilbert space (RKHS) with accompanying induced RKHS norm \(\|\cdot\|_K\). Given noisy observations $(\mathbf{X}, \mathbf{Y})$, the corresponding regularized optimal recovery problem is to find the function $\hat{f}\in\mathcal{H}_K$ that minimizes the sum of the squared RKHS norm and the scaled data-misfit term:

\[
\mathcal{J}(\mathbf{X};\mathbf{Y})=\min_{g\in\mathcal{H}_K}\|g\|_K^2 + \frac{1}{\sigma_\varepsilon^2}\|g(\mathbf{X})-\mathbf{Y}\|_2^2.
\]
By the Representer Theorem, the minimizer $\hat{f}$ has a finite-dimensional representation in terms of kernel evaluations at the training inputs. In particular, $\hat{f}$ evaluated at a new point can be written as 
\[\hat{f}(\cdot)=K(\cdot, \mathbf{X})(K(\mathbf{X}, \mathbf{X}) + \sigma_\varepsilon^2 I)^{-1}\mathbf{Y}.\]
This estimator coincides with the posterior mean of Gaussian process regression with covariance kernel $K$ and independent Gaussian observation noise of variance $\sigma_\varepsilon^2$. The corresponding minimum value of the objective \(\mathcal{J}(\mathbf{X};\mathbf{Y})\)  is
\[
\mathcal{J}(\mathbf{X};\mathbf{Y}) = \mathbf{Y}^\top\left(K(\mathbf{X},\mathbf{X}) + \sigma_\varepsilon^2 I\right)^{-1}\mathbf{Y},
\]
proved in \cite{propp2026discovery}.

\section{Method part 1: Data-driven $H(\mathrm{div})$-conforming reduced space construction}\label{sec:method1}
\label{sec:ddwf}
In this section, we construct the structure-preserving reduced spaces used by the surrogate. The goal is to learn the data-driven Whitney forms from a trainable partition of unity (PoU) conditioned on $Z$. The PoU functions define coarse control volumes, represented by 0-forms, while the associated 1-forms encode generalized fluxes between them (\textbf{\Cref{subsec:pou_whitney}}). This construction yields reduced scalar and flux spaces compatible with the \(H(\mathrm{div})\)--\(L^2\) structure introduced above, allowing the resulting $H(\mathrm{div})$-conforming finite element spaces to adapt to the conditioning variable $Z$ (\textbf{\Cref{subsec:coarse_operator,subsec:surjective}}). Finally, we connect the induced coarse spaces to a graph interpretation in \textbf{\Cref{subsec:graph_cal}} that will support the optimal recovery formulation.

\subsection{PoU Whitney forms and coarse spaces}\label{subsec:pou_whitney} We construct the learnable coarse spaces from the fine-scale $\mathrm{RT}_0$/$dgP_0$ discretization. Consider the fine-scale piecewise-constant scalar space $Q_h^f$ and the corresponding fine-scale lowest-order $RT_0$ flux space $V_h^f$ defined as
\begin{equation}\label{eq:fine_space}
Q_h^{f}\coloneqq\operatorname{span} \{\phi_a^{P_0}:a\in\mathcal C_h\},\qquad V_h^f
\coloneqq
\operatorname{span}
\{\phi_e^{RT}:e\in \mathcal{E}_h\},
\end{equation} where $\mathcal{C}_h=\{1,\ldots,N_{\mathrm{cell}}\}$ denotes the set of fine cell indices and $\mathcal{E}_h$ is the set of edges for fine $RT_0$ flux degrees of freedom. The basis \(\{\phi_a^{P_0}\}\) is chosen so that $$\sum_{a\in\mathcal C_h}\phi_a^{P_0}(x)=1.$$

In this work, we only use the 0-form and 1-form components needed for the $H(\mathrm{div})-L^2$ tail. Furthermore, instead of directly constructing the Whitney forms using the classic barycentric interpolant, we employ neural networks to construct neural Whitney forms that still possess the desired PoU property. 

Let $N_0$ denote the number of partitions. For a fixed conditioning variable, let $W=[W_{ia}]\in\mathbb{R}^{N_0\times N_{cell}}$ be a learnable PoU weight matrix satisfying:
\begin{equation}\label{eq:W}
  W_{ia}\ge 0,\qquad \sum_{i=1}^{N_0}W_{ia}=1\quad \text{for every } a\in \mathcal{C}_h.
\end{equation} 
Then the coarse 0-forms and 1-forms are constructed as follows.
\begin{proposition}[Coarse 0-forms]\label{prop:coarse_0forms}
For a fixed conditioning variable, let
$W$ be the learned PoU weights as in \Cref{eq:W}. The coarse 0-form is defined as
\begin{equation}
\psi_i^0=\sum_{a\in\mathcal C_h}
W_{ia}\phi_a^{P_0},
\qquad i=1,\ldots,N_0,
\label{eq:coarse_0form}
\end{equation}
which forms a partition of unity:
\begin{equation}
\sum_{i=1}^{N_0}\psi_i^0=\sum_{a\in\mathcal C_h}\left(\sum_{i=1}^{N_0}W_{ia}\right)\phi_a^{P_0}
=1.    
\end{equation}
The corresponding coarse $P_0$ space is defined as:
\begin{equation}
    Q^c_h \coloneqq\mathcal{W}^0= \mathrm{span}\{\psi^0_i\}_i\subset Q_h^f.
\end{equation}
\end{proposition}
Next, we construct the coarse 1-forms that still follow the same PoU Whitney principle with a special design of separating the interior edges and boundary edges. This separation allows us to maintain the boundary information more efficiently during the coarsening procedure, especially for scenarios where the geometry is complex or where boundary conditions are different for boundary subdomains.
\begin{proposition}[Interior and boundary coarse 1-forms]\label{prop:coarse_1forms} For a fixed conditioning variable, let
$W$ be the learned PoU weights as in \Cref{eq:W}. Let $e\in\mathcal{E}_{\mathrm{int}}$ denote the interior fine edges with two adjacent cells $K_L(e),\ K_R(e)$. For each pair of coarse 0-forms $(i,j)$ with $1\le i< j\le N_0$, the coarse interior 1-form is defined as
\begin{equation}\label{eq:psi1_int}
  \boldsymbol{\psi}_{ij}^{1,\mathrm{int}}(x)
  \;=\; \sum_{e \in \mathcal{E}_{\mathrm{int}}}
    \Bigl(
      W_{i,K_{\mathrm{L}}(e)}\,W_{j,K_{\mathrm{R}}(e)}
    - W_{j,K_{\mathrm{L}}(e)}\,W_{i,K_{\mathrm{R}}(e)}
    \Bigr)\,\boldsymbol{\phi}_e^1(x).
\end{equation}
Let the boundary fine edges be partitioned into disjoint groups
$$\mathcal E_{\rm bc}
=
\mathcal E_1\cup\cdots\cup\mathcal E_r,
\qquad
\mathcal E_\gamma\cap\mathcal E_{\gamma'}=\emptyset
\quad(\gamma\ne\gamma'),$$
where each group may represent a distinct boundary component or boundary type. Let $e\in\mathcal{E}_{\mathrm{\gamma}}$ denote the boundary fine edges with one adjacent cell $K(e)$. For $\gamma=1,\dots,r$ and $i = 1, \dots, N_0$, the coarse boundary 1-form is defined as
    \begin{equation}\label{eq:psi1_bc}
      \boldsymbol{\psi}_{i,\gamma}^{1,\mathrm{bc}}(x)
      \;=\; \sum_{e \in \mathcal{E}_\gamma}
        W_{i,K(e)}\,\boldsymbol{\phi}_e^1(x).
    \end{equation}
Concatenating the interior and boundary 1-forms, we have:
\begin{equation*}
    \boldsymbol{\psi}^1=[\boldsymbol{\psi}^{1,\mathrm{int}}\quad\boldsymbol{\psi}^{1,\mathrm{bc}}],
\end{equation*}
where $N_1=N_1^{int}+N_1^{bc}=\binom{N_0}{2}+rN_0$. The corresponding coarse $RT_0$ space is 
\begin{equation}
    V_h^c\coloneqq\mathcal W^1=\operatorname{span}\{\boldsymbol\psi_\alpha^1\}_{\alpha=1}^{N_1}
\subset V_h^f.
\end{equation}
\end{proposition}
\begin{remark}\label{remark:coarse_1forms}
The classic PoU Whitney 1-form between two coarse partitions is
\[
    \psi_{ij}^{1}
    =
    \lambda_i \nabla\lambda_j-\lambda_j \nabla\lambda_i .
\]
The dense representation applies a double sum over all oriented cell pairs:
\[
\psi_{ij}^{1}=\sum_{a,b}
    W_{ia}W_{jb}\boldsymbol\phi_{ab}^{RT_0},
\]
assuming that the oriented flux basis $\boldsymbol\phi_{ab}^{RT_0}$ is defined for each ordered pair $(a,b)$. In the actual finite element implementation, however, $\mathrm{RT}_0$ basis functions are associated only with fine mesh edges. By the definition, we have
\[\boldsymbol\phi_{K_L(e),K_R(e)}^{RT_0}=\boldsymbol\phi_e^{RT_0},\qquad
    \boldsymbol\phi_{K_R(e),K_L(e)}^{RT_0}=-\boldsymbol\phi_e^{RT_0}.\]
Thus \Cref{prop:coarse_1forms} is the edge-restricted construction of the same antisymmetric principle, and the coefficient for the interior coarse 1-forms in \Cref{eq:psi1_int} should be considered the discrete mesh-edge analogue of the classical Whitney representation. For later propositions and proofs, we use the double sum notation.
\end{remark}

\subsection{Coarse operators}\label{subsec:coarse_operator}
For convenience, we hereafter denote by \(\{\boldsymbol\psi_\alpha^1\}_{\alpha=1}^{N_1}\) the full set of interior and boundary coarse 1-form basis functions. With the coarse 0-forms $\{\psi_i^0\}_{i=1}^{N_0}$ and 1-forms $\{\psi_\alpha^1\}_{\alpha=1}^{N_1}$, we can further define the corresponding coarse operators.
\begin{proposition}[Coarse operators]\label{prop:coarse_operators}For $i,j = 1,\dots,N_0,$ and $\alpha,\beta=1,\dots,N_1$,  the coarse mass matrices and divergence matrix are defined by
\begin{align*}
  (\mathbf{M}_0)_{ij}
  = (\psi_i^0,\ \psi_j^0)_\Omega,
  \quad (\mathbf{M}_1)_{\alpha\beta}
  = (\boldsymbol{\psi}_{\alpha}^1,\
     \boldsymbol{\psi}_{\beta}^1)_\Omega,
  \quad \mathbf{D}_{i\alpha}=
  (\nabla\!\cdot\!\boldsymbol{\psi}_\alpha^1,\psi_i^0)_\Omega.
\end{align*}
\end{proposition}
The above construction for coarse bases and operators is purely algebraic and therefore remains valid for any $W$ satisfying the property in \Cref{eq:W}. For a conditioning variable $z$, the neural coarsening map gives $W=W_\theta(z)$ and the resulting conditional operators in reduced space are
$Q_h^c(z;\theta),\ V_h^c(z;\theta),\  D_\theta(z)$.
In practice, for $N$ training instances, $W$ can be a tensor of shape $W\in\mathbb{R}^{N\times N_0\times N_{\rm cell}}$, where $W^{(n)}\in\mathbb{R}^{N_0\times N_{\rm cell}}$ represents the conditional neural Whitney form for sample $n$ given a conditioning variable $z^{(n)}$. The choice of coarse forms may vary depending on different factors, such as the treatment of boundaries or the choices of finite element spaces, but the fundamental principle is that by introducing the trainable $W$, we are able to adapt the PoU flexibly with respect to certain conditioning variables. Moreover, the learned coarse spaces inherit the \(H(\mathrm{div})\)--\(L^2\) structure from the fine space with the coarse divergence operator as in the following diagram:
\begin{equation}
    \begin{array}{ccc}H(\mathrm{div},\Omega) & \xrightarrow{\;\nabla\!\cdot\;} & L^2(\Omega) \\[0.4em]
\cup &  & \cup \\[-0.2em]
V_h^c(z;\theta) & \xrightarrow{\;D_\theta(z)\;} & Q_h^c(z;\theta).
\end{array}
\end{equation}

The following proposition provides a sufficient condition for the full row rank of the coarse divergence matrix. This condition is a nondegeneracy condition of the learned coarse basis and can be checked from the learned PoU weights.
\begin{proposition}[Rank condition]\label{prop:rank} Let the boundary fine edges be partitioned as in \Cref{prop:coarse_1forms}. Assume that the fine boundary basis $\boldsymbol\phi_e^1$ for boundary edge $e\in\mathcal E_\gamma$ is oriented so that
$\beta_e^\gamma \coloneqq (\nabla\!\cdot\boldsymbol\phi_e^1,\phi_{K(e)}^{P_0})_\Omega>0$. Let \(D_\gamma\in\mathbb R^{N_0\times N_0}\) denote the block of the coarse divergence matrix \(D\) whose columns correspond to the boundary basis functions associated with $\mathcal E_\gamma$. Then 
$$(D_\gamma)_{i\alpha}=\sum_{e\in\mathcal E_\gamma}\beta_e^\gamma W_{\alpha,K(e)}W_{i,K(e)}=(W_\gamma B_\gamma W_\gamma^\top)_{i\alpha},$$
where $$(W_\gamma)_{i e}=W_{i,K(e)},\qquad B_\gamma=\operatorname{diag}\{\beta_e^\gamma:e\in\mathcal E_\gamma\}.$$ 
If $\sum_{\gamma=1}^r W_\gamma B_\gamma W_\gamma^\top$ is positive definite, then the coarse divergence matrix has full row rank, $\text{rank}(D)=N_0$.
\end{proposition}
\begin{proof}
    See Appendix \ref{app:proof_prop_rank}.
\end{proof}
\subsection{Surjectivity property under $H(\mathrm{div})$ setting and flux conservation}\label{subsec:surjective}
The divergence operator $\nabla\cdot$ is surjective from the Raviart--Thomas space $\mathrm{RT}_k(K)$ onto $P_k(K)$ in the standard fine scale~\cite{boffi2013mixed}. In particular, for $\mathrm{RT}_0$/$dgP_0$, the divergence represents the cellwise flux conservation. Here, we extend this surjectivity property and conservation structure to coarse operators induced by the Whitney forms construction.

\begin{proposition}\label{prop:surjective}
Let $\mathcal{W}^1 \coloneqq \mathrm{span}\{\psi^{1}_{a}\}_{a=1}^{N_1},\ \mathcal{W}^0 \coloneqq \mathrm{span}\{\psi^0_i\}_{i=1}^{N_0}$ be the coarse flux and scalar spaces following \Cref{prop:coarse_0forms} and \Cref{prop:coarse_1forms}. Then: 
\begin{equation}
\nabla\!\cdot \mathcal{W}^1 \subseteq \mathcal{W}^0.
\end{equation}
Moreover, if the rank condition in \Cref{prop:rank} holds, then the coarse divergence matrix has full row rank and the coarse divergence is surjective from \(\mathcal{W}^1\) onto \( \mathcal{W}^0\).
\end{proposition}

\begin{proof} Recall that we have fine-scale basis functions $\{\boldsymbol\phi^1\}\subset RT_0$ and $\{\phi^0\}\subset P_0$. Given the fine-scale divergence identity for $(\mathrm{RT}_0, \mathrm{dgP}_0)$, we have
\begin{equation}
    \nabla\!\cdot\boldsymbol\phi_e^{RT_0}=
\phi_{K_R(e)}^{P_0}-\phi_{K_L(e)}^{P_0}.
\end{equation}
For interior edge $e=(i,j)$, we can derive the identity in the coarse scale by substituting the definition of coarse basis functions:
\begin{align*}
    \nabla\!\cdot\psi_{ij}^1
    &=
    \nabla\!\cdot
    \left(  \sum_{a,b}W_{ia}W_{jb}\phi_{ab}^{RT_0}
    \right)  \\
    &=
    \sum_{a,b}W_{ia}W_{jb}\phi_b^{P_0}
    -
    \sum_{a,b}W_{ia}W_{jb}\phi_a^{P_0}\\
    &= \left(\sum_a W_{ia}\right)\psi_j^0
    -\left(\sum_b W_{jb}\right)
    \psi_i^0\\
    &=m_i\psi_j^0-m_j\psi_i^0\in\mathcal{W}^0.
\end{align*}
Similarly, by linearity, the boundary coarse 1-form is also a linear combination of fine $P_0$ basis functions and belongs to the span of the coarse scalar space. Therefore, $$\nabla\!\cdot\mathcal W^1\subseteq \mathcal W^0.$$ 
Next we prove the surjectivity of the coarse divergence operator. Let
\begin{equation}
\boldsymbol v=\sum_{a=1}^{N_1}
\widehat{F}_a\psi_a^1\in\mathcal{W}^1,
\qquad q=\sum_{i=1}^{N_0}\widehat{q}_i\psi_i^0\in\mathcal{W}^0.
\end{equation}
By linearity,
\begin{equation}
\nabla\cdot\boldsymbol v=\sum_{a=1}^{N_1}
\widehat{F}_a\nabla\cdot\psi_a^1\in\mathcal{W}^0.
\end{equation}
Then testing against $\psi_i^0$ for $i=1,\dots, N_0$ and expanding for coarse variables gives:
\begin{align}
    (\nabla\cdot v,\psi_i^0)_\Omega &=\left(\nabla\cdot\left(\sum_{a=1}^{N_1}\widehat{F}_a\boldsymbol\psi_a^1\right),\psi_i^0\right)_\Omega\\
    &=\sum_{a=1}^{N_1}\widehat{F}_a
\left(\nabla\cdot\boldsymbol\psi_a^1,\psi_i^0\right)_\Omega=\sum_{a=1}^{N_1}
\widehat{F}_a D_{ia }\\
(q,\psi_k^0)_\Omega&=
\sum_{i=1}^{N_0}
q_i(\psi_i^0,\psi_k^0)_\Omega=\sum_{i=1}^{N_0}
q_i (M_0)_{ik},
\end{align}
where $D$ and $M$ are as defined in \Cref{prop:coarse_operators}. When \(\operatorname{rank}(D)=N_0\), the map $D:\mathbb R^{N_1}\to\mathbb R^{N_0}$
is onto and thus for arbitrary $q$, there exist coefficients $\widehat{F}\in\mathbb{R}^{N_1}$ such that $D\widehat{F}=M_0q$. For this choice of $\widehat{F}$, the corresponding $\boldsymbol v\in\mathcal{W}^1$ satisfies
\begin{equation}\label{eq:proj_surjective}(\nabla\!\cdot\boldsymbol v,\psi_k^0)_\Omega=(q,\psi_k^0)_\Omega.
\end{equation}
Thus we conclude that the coarse divergence operator is surjective.
\end{proof}

Moreover, since the coefficient in \eqref{eq:psi1_int} is the discrete analogue of the classical Whitney exchange form \(\lambda_i\nabla\lambda_j-\lambda_j\nabla\lambda_i\), the resulting antisymmetry property \(\nabla\!\cdot \psi^{1}_{ij} = -\,\nabla\!\cdot \psi^{1}_{ji}\) with respect to the exchange of indices ensures by construction that the 1-forms encode equal and opposite fluxes exchanged between 0-forms. Summing over all indices cancels the internal fluxes, so that only the boundary contributions remain. This recovers a discrete conservation law consistent with the divergence theorem and shows that the learned coarse construction preserves the flux balance. 

\subsection{Graph calculus and circuit interpretation for FEEC}\label{subsec:graph_cal}
The preceding construction in \Cref{subsec:pou_whitney} and \Cref{subsec:coarse_operator} produces an $H(\mathrm{div})$-conforming reduced space where conservation is preserved by the coarse divergence operator. To formulate the GP optimal recovery problem in the next section, we now reinterpret the reduced subcomplexes as a graph-valued circuit model and connect coarse operators with the graph calculus. This interpretation bridges the structure-preserving FEEC with the graph-based GP model used to learn the nonlinearity and perform uncertainty quantification later. The viewpoint is related to the computational graph completion (CGC)~\cite{owhadi2022computational} problem, where the unknown functions and unobserved variables can be recovered on a given graph. In contrast, the graph in this work is induced by the coarse finite element subcomplexes, which can be considered an autonomous identification of a dense graph.

Let $\mathcal{G} = (\mathcal{V},\mathcal{E})$ be a graph with $\lvert \mathcal{V} \rvert=N_0$ vertices and $\lvert \mathcal{E} \rvert=N_1$ edges. The coarse spaces constructed via the partition of unity Whitney forms naturally define a graph where each vertex $i \in \mathcal{V}$ corresponds to a coarse 0-form $\psi^0_i$ and each edge $e \in \mathcal{E}$ corresponds to a coarse $1$-form $\psi^1_{e}$.
We define the cochain spaces as:
\[
C^0(\mathcal{G}) \coloneqq \mathbb{R}^{\lvert V \rvert}, 
\qquad
C^1(\mathcal{G}) \coloneqq \mathbb{R}^{\lvert \mathcal{E} \rvert}.
\]
Then the coarse variables admit the interpretation on vertices and edges, respectively:
\[
\widehat{\mathbf u}\in C^0(\mathcal G),
\qquad
\widehat{\mathbf F}\in C^1(\mathcal G),
\]
where $\widehat{\mathbf u}$ represents node states and $\widehat{\mathbf F}$ represents oriented fluxes or currents.

In order to formulate the mapping between nodal values and edge values, we introduce the edge-vertex signed incidence matrix $B_0\in\mathbb{R}^{E\times V}$ defined by 
\begin{equation}
(B_0)_{i,j}=
\begin{cases}
    -1 & \text{if edge }e_{i} \text{ leaves vertex }v_j,\\
    1 & \text{if edge }e_{i} \text{ enters vertex }v_j,\\
    0&\text{otherwise}.
\end{cases}
\end{equation}
Subsequently for an oriented edge $e=(i,j)$, we can define the discrete graph gradient operator and graph divergence as:
\begin{align*}
\text{(\textbf{Graph gradient})}&\quad(\delta_0\widehat{\mathbf u})(e)=\widehat{ u}_j-\widehat{ u}_i, \\
\text{(\textbf{Graph divergence})}&\quad(\delta_0^\top\widehat{\mathbf{F}})(v)=\sum_{v_i\sim v} \widehat{F}(v_i, v)-\sum_{v\sim v_j}\widehat{F}(v, v_j).
\end{align*}
The discrete graph gradient, mapping the nodal values to edge values, captures the difference across edges. Conversely, the graph divergence, mapping edge values back to values at nodes, measures the net flux entering or leaving the node $v$ over its adjacent edges.
Let $D_{\mathcal{G}} \in \mathbb{R}^{|\mathcal V|\times |\mathcal E|}$
denote the graph divergence matrix, defined by
\[
(D_{\mathcal{G}} \widehat{\mathbf F})_i
=
\sum_{e\in\mathcal{E}} D_{ie}\widehat{F}_{e}.
\]
With the construction of coarse Whitney forms in \Cref{subsec:coarse_operator}, $D$ has the same incidence structure as the graph divergence operator, possibly weighted by the mass or basis normalization depending on the specific construction. Thus, this gives the conservation law in the reduced spaces and supplies the finite-dimensional variables on which the GP acts.

\section{Method part 2: Optimal recovery of physics with quantified uncertainty}\label{sec:method2}
In this section, we develop a probabilistic representation of physical fields on the graph induced by the coarse spaces. This provides a unified framework where the physical law acts on discrete structures while uncertainty is quantified in a functionally consistent manner.

\subsection{Structure-preserving Optimal Recovery Problem}\label{subsec:opt_formulation}
We now introduce a GP optimal recovery formulation for learning the nonlinear flux correction in the coarse $H(\mathrm{div})-L^2$ space, consistent
with the FEEC structure established above. For each coarse edge $e=(i,j) \in \mathcal{E}$, we can define the local input with different choices:
\[
\hat{\mathbf{u}}_e \coloneqq
\begin{cases}
(\hat{u}_i, \hat{u}_j) \in \mathbb{R}^2, & \text{(endpoint representation)}, \\[4pt]
\delta_0 \hat{u}_e = \hat{u}_j - \hat{u}_i \in \mathbb{R}, & \text{(edge gradient representation)}.
\end{cases}
\]
For training instances $n=1,\ldots,N$, the desired GP input features become
\[x_e^{(n)}
=(\widehat{\mathbf u}_e^{(n)},
z^{(n)},
s^{(n)}).\]
Here $\widehat{\mathbf u}_e^{(n)}$ is the coarse state variable associated with edge $e$,  $z^{(n)}$ denotes the conditioning variable, and $s^{(n)}$ denotes the geometric embedding. 

In the independent kernel construction, for a fixed edge $e$, we collect the training inputs as $X_e=\{x_e^{(n)}\}_{n=1}^N$ and the corresponding global kernel is the block-diagonal matrix defined as 
\begin{equation}
   K_e\coloneqq K_e(X_e,X_e)\in\mathbb R^{N\times N}, \qquad K=\text{diag}\left(K_1,\ldots,K_{N_1}
\right)
\in
\mathbb R^{N_1N\times N_1N}.
\end{equation}
The GP model on edge $e$ maps local features to the flux on the edge:
\begin{equation}
    f_e(X_e) \;\mapsto\; \widehat{F}^{\rm gp}_e,\qquad \widehat{\mathbf F}^{\rm gp}=\text{vec}\left(\left[\widehat F_e^{\rm gp,(n)}\right]\right)\in\mathbb R^{N_1N},
\end{equation}
where each $f_e$ belongs to a reproducing kernel Hilbert space $\mathcal{H}_{K_e}$ and $\widehat{\mathbf F}^{\rm gp}$ is the vectorized global output over all edges. 

Instead of assigning independent GPs to each edge, we may also stack all edge samples into a global matrix 
\begin{equation}
    X=\{x_e^{(n)}:e=1,\ldots,N_1,\ n=1,\ldots,N\}
\end{equation}
and construct a single shared GP across all cochains:
\[
K=
\left[
k_\theta\!\left(x_e^{(n)},x_{e'}^{(m)}\right)
\right]_{(e,n),(e',m)}\in \mathbb R^{N_1N\times N_1N}
\]
where $k_\theta$ denotes the covariance kernel. Since both cases use a large global matrix, for simplicity we denote the global matrix as $K$ in the following.

Recall the governing equation for the physical law in \Cref{eq:flux-ansatz}. By restricting the PDE to data-driven Whitney spaces 
$\mathcal{W}^0(\Omega;\theta)$ and $\mathcal{W}^1(\Omega;\theta)$, the goal is to find $(u,\mathbf{F}) \in \mathcal{W}^0(\Omega;\theta)\times \mathcal{W}^1(\Omega;\theta)$ such that
\begin{align}
(\mathbf{F}, \mathbf{v})_\Omega
&= (u, \nabla \cdot\mathbf{v})_\Omega + (\mathcal{N}[u], \mathbf{v})_\Omega- \langle u_D, \mathbf{v}\cdot\mathbf{n} \rangle_{\Gamma_D}, &&\qquad \mathbf{v} \in \mathcal{W}^1(\Omega;\theta)
\\
(\nabla\cdot \mathbf{F}, q)_\Omega 
&= (f,q)_\Omega,&&\qquad q\in  \mathcal{W}^0(\Omega;\theta),
\end{align}
where $f$ is the source term and $\langle u_D, \mathbf{v}\cdot\mathbf{n} \rangle_{\Gamma_D}
\;=\;
\int_{\Gamma_D} u_D \, (\mathbf{v}\cdot\mathbf{n}) \, ds$ is the weak Dirichlet boundary condition. Recall that $\widehat{\mathbf u}$ denotes the coarse scalar coefficient. Let $\widehat{\mathbf F}^{\rm gp}$ denote the nonlinear flux correction on the coarse edge and $\widehat{\mathbf F}$ denote the total coarse flux coefficient. Using the FEEC structure, the problem can be further written as
\begin{align}
\mathbf{M}_1\widehat{\mathbf{F}} 
&= D^\top \widehat{\mathbf u} 
+ \mathbf{M}_1 \widehat{\mathbf F}^{\rm gp}-\widehat{\mathbf g}_D,
\\
D\widehat{\mathbf{F}} 
&= \widehat{\mathbf f},
\end{align}
where $\widehat{\mathbf f}$ is the coefficient vector of the source term and $\widehat{\mathbf g}_D$ represents the Dirichlet boundary contribution with  $\widehat{\mathbf g}_D=\int_{\Gamma_D} u_D \boldsymbol{\psi}^1\cdot \boldsymbol{n} \, ds$. $\mathbf{M}_1$ and $D$ are the discrete mass matrix and divergence operator as defined in~\Cref{prop:coarse_operators}. For simplicity, we let $\theta=(\theta_{GP},\theta_W)$ 
denote the set of trainable parameters for the GP model and neural Whitney forms. Eliminating $\widehat{\mathbf{F}}$ gives the reduced residual:
\begin{equation}R_{\theta}(\widehat{\mathbf u},\widehat{\mathbf F}^{\rm gp};z):=
D\mathbf{M}_1^{-1} D^\top \widehat{\mathbf u}
+
D\widehat{\mathbf F}^{\rm gp}-
\widehat{\mathbf{b}}, 
\qquad \text{with }\widehat{\mathbf{b}}=\widehat{\mathbf f}+D\mathbf{M}_1^{-1}\widehat{\mathbf g}_D.
\end{equation}

Finally, we arrive at the full optimization problem, combining the GP recovery and the data-driven Whitney forms on the graph:
\begin{equation}\label{eq:opt}
\begin{split}
	\min_{\theta, \widehat{\mathbf{u}}} \min_{\widehat{\mathbf F}^{\rm gp}} & \quad (\widehat{\mathbf F}^{\rm gp} )^\top(K(X,X) + \sigma_{\varepsilon}^2 I)^{-1} \widehat{\mathbf F}^{\rm gp} + \log \det (K(X,X) + \sigma_{\varepsilon}^2 I)\\
    &\quad +\sum_{n=1}^N
	\left\|
	\mathbf u_{\rm data}^{(n)}
	-
	\mathbf u_{\rm approx}^{(n)}
	\right\|^2 + \sum_{n=1}^N
	\left|
	\mathbf F_{\Gamma,{\rm data}}^{(n)}
	-
	\mathbf F_{\Gamma, \rm approx}^{(n)}
	\right|^2\\
	\text{s.t.} \qquad&      \mathcal R_{\theta}
\left(\widehat{\mathbf u}, \widehat{\mathbf F}^{\rm gp}; z^{(n)}
\right)
=0,\\
\end{split} \end{equation}
where $\mathbf{u}_{\text{approx}}, \ \mathbf{F}_{\text{approx}}$ are the reconstruction of the solution field and flux back to the original fine-scale space. The first two terms promote smoothness in the RKHS norm and penalize overly complex models to avoid overfitting. Moreover, as the recovery problem is now defined for the coarse coefficients $\hat{u}$ and $\widehat{\mathbf F}^{\rm gp}$, there is no guarantee that the solution and flux projected back to the original space can reconstruct the target solution $\mathbf{u}_{\text{data}}$ and $\mathbf{F}_{\text{data}}$ as desired. Thus, we introduce two additional terms consisting of the mean squared error (MSE) reconstruction loss in the fine scale.

The reconstruction of the solution in \eqref{eq:opt} can be directly obtained by mapping the coarse $0$-cochain coefficients back to the fine scale with the neural Whitney bases: 
\begin{equation}  \mathbf{u}_\text{approx}=\widehat{\mathbf u}^\top\Psi_0.
\end{equation}
The total ground-truth flux and reconstructed flux on the target boundary $\Gamma$ from the PoU Whitney form and coarse coefficient are:

\begin{equation}
  \mathbf{F}_{\Gamma, \text{data}} = \sum_{e\in\Gamma} F_{e},\qquad\mathbf{F}_{\Gamma,\text{approx}}=\sum_{i=1}^{Npou}\sum_{e\in\Gamma}W_{ie}\widehat{F}_i,
\end{equation}
where $W_{ie}$ denotes the corresponding weight for the $i$-th PoU and fine-scale edge $e$ and $\widehat{\mathbf F}=\mathbf{M}_1^{-1}D^\top \hat{u} +\widehat{\mathbf F}^{\rm gp} - \mathbf{M}_1^{-1}\mathbf{g}_D$.
  
\subsection{Bilevel optimization}
\label{subsec:bilevel}
The optimization problem is nested and strongly coupled because the GP inputs depend on the unknown coarse potentials $\widehat{\mathbf u}_e$. Therefore, we apply a bilevel optimization procedure to train the transformer for data-driven Whitney forms and the GP model for flux prediction simultaneously. The bilevel training algorithm is summarized in \Cref{alg:bilevel}.

\begin{algorithm}
\caption{Bilevel training for structure-preserving optimal recovery}
\label{alg:bilevel}
\small
\begin{algorithmic}[1]
\REQUIRE  \(\mathcal D=\{(z^{(n)},u^{(n)}_{\text{data}},F^{(n)}_{\text{data}})\}_{n=1}^N\); fine operators; PoU network \(W_{\theta_W}\); GP kernel \(K_{\theta_{\mathrm{GP}}}\)
\REQUIRE Optimizers: Adam for \(\theta_{\mathrm{GP}}\), Shampoo for \(\theta_W\)
\STATE Initialize \(\widehat{\mathbf u}\), \(\theta_W\), and 
\(\theta_{\mathrm{GP}}\)
\FOR{epoch \(=1,\ldots,T\)}

    \STATE \textbf{PoU forward:} Compute \(W\gets W_{\theta_W}(z)\)
    \STATE Compute coarse bases \(\Psi^0,\Psi^1\), and operators \(M_1,D,R_\Gamma,\widehat{\mathbf b}\)
    \STATE Form coarse GP features $X=\{x_e^{(n)}\}_{e,n},\
        \widetilde K
        =
        K_{\theta_{\mathrm{GP}}}(X,X)+\sigma_\varepsilon^2 I$

    \STATE \textbf{\underline{Inner problem}}: Fix \(\widehat{\mathbf u}\), \(\theta_W\), \(\theta_{\mathrm{GP}}\), solve the constrained recovery problem for \(\widehat{\mathbf F}^{\rm gp,\star}\)
    \STATE \quad \textbf{Update} $        \widehat{\mathbf F}^{\rm gp,\star}
        \leftarrow
    \operatorname{SolveKKT}
    \left(
    \widetilde{\mathbf K},D,R_\Gamma,\widehat{\mathbf b},F_{\Gamma,{\rm data}}
    \right)$
    \STATE \quad \textbf{Update} $\theta_{\mathrm{GP}}
        \leftarrow
        \operatorname{Adam}
        \left(
            \theta_{\mathrm{GP}},
            \nabla_{\theta_{\mathrm{GP}}}\mathcal L_{\mathrm{inner}}\right)$

    \STATE \textbf{\underline{Outer problem}}: Fix \(\widehat{\mathbf F}^{\rm gp, \star}\), \(\theta_{\mathrm{GP}}\), solve for $\widehat{\mathbf u}^\star$ with least squares method
    \STATE \quad \textbf{Update} \(\widehat{\mathbf u}^\star\leftarrow
\operatorname{lstsq}
\left(
DM_1^{-1}D^\top,\,
\widehat{\mathbf b}-D\widehat{\mathbf F}^{\rm gp, \star}
\right)\)
    \STATE \quad \textbf{Update} $\theta_{W}\leftarrow
        \operatorname{Shampoo}
        \left(\theta_W,
            \nabla_{\theta_W}
            \mathcal L_{\mathrm{outer}}
        \right)$
    
\ENDFOR
\STATE \RETURN trained \(\theta_W^\star,\theta_{\mathrm{GP}}^\star,\widehat{\mathbf u}^\star\), and \(\widehat{\mathbf F}^\star\).
\end{algorithmic}
\end{algorithm}

For the inner loop, we first fix $\widehat{\mathbf u}$ and $\theta_W$ and solve for the closed-form solution $\widehat{\mathbf F}^{\rm gp}$ via the Karush--Kuhn--Tucker (KKT) system. Since the flux observations are defined on fine boundary edges, for simplicity, we introduce the coarse-to-fine reconstruction operator $R_\Gamma(W): \mathbb{R}^{N_1} \to \mathbb{R} $ that maps the coarse flux degrees of freedom ($\hat{\mathbf F}$) back to the total flux through the boundary region $\Gamma$. 
Substituting $\widehat{\mathbf{F}}^{\rm gp}$ gives the inner optimization problem as
\begin{align}\label{eq:opt_inner}
\min_{\widehat{\mathbf F}^{\rm gp}}\;
&\left\|
\mathbf{F}_{\Gamma, \text{data}}
-R_\Gamma(W)
\left(
\mathbf{M}_1^{-1}D^\top \widehat{\mathbf u} +\widehat{\mathbf F}^{\rm gp} - \mathbf{M}_1^{-1}\mathbf{g}_D
\right)
\right\|_2^2 \\
&\quad
+
(\widehat{\mathbf F}^{\rm gp})^\top (K+\sigma_{\varepsilon}^2 I)^{-1}\widehat{\mathbf F}^{\rm gp}
+
\log\det(K+\sigma^2 I) \\
\text{s.t.}\quad
&D \widehat{\mathbf F}^{\rm gp}
=
\mathbf{b} - D \mathbf{M}_1^{-1} D^\top \widehat{\mathbf u}.
\end{align}
Since the term $\log\det(K+\sigma_{\varepsilon}^2 I)$ does not depend on $\widehat{\mathbf F}^{\rm gp}$, it can be omitted when solving for $\widehat{\mathbf F}^{\rm gp}$. Introducing the Lagrange multiplier $\lambda$, the optimality system becomes:
\[
\begin{bmatrix}
\tilde{K}^{-1}
+
R_\Gamma^\top R_\Gamma
&
D^\top \\
D & 0
\end{bmatrix}
\begin{bmatrix}
\widehat{\mathbf F}^{\rm gp} \\
\lambda
\end{bmatrix}
=
\begin{bmatrix}
R_\Gamma^\top
\Bigl(
\mathbf{F}_{\Gamma, \text{data}}
- R_\Gamma
\left(
\mathbf{M}_1^{-1}D^\top \widehat{\mathbf u} - \mathbf{M}_1^{-1}\mathbf{g}_D
\right)
\Bigr)
\\
\mathbf{b} - D \mathbf{M}_1^{-1} D^\top \widehat{\mathbf u}
\end{bmatrix}.
\]
where $\tilde{K}= K + \sigma^2 I$. Solving by the Schur complement, we obtain the closed-form solution for optimal $\widehat{\mathbf F}^{\rm gp}$ at the current stage:
\begin{align}
  \lambda
&=
\bigl(D \left(\tilde{K}^{-1} + R_\Gamma^\top R_\Gamma\right)^{-1} D^\top\bigr)^{-1}
\bigl(D \left(\tilde{K}^{-1} + R_\Gamma^\top R_\Gamma\right)^{-1} r_t - r_b\bigr),\\
\widehat{\mathbf F}^{\rm gp,\star}
&=
\left(\tilde{K}^{-1} + R_\Gamma^\top R_\Gamma\right)^{-1}\bigl(r_t - D^\top \lambda\bigr),
\end{align}
where $r_t=
R_\Gamma(W)^\top \Bigl(
\mathbf{F}_{\Gamma, \text{data}}-R_\Gamma\bigl(
\mathbf{M}_1^{-1}D^\top \widehat{\mathbf u} - \mathbf{M}_1^{-1}\mathbf{g}_D\bigr)
\Bigr)$ and $r_b = \mathbf{b} - D \mathbf{M}_1^{-1} D^\top \widehat{\mathbf u}.$ 

Then with the optimal $\widehat{\mathbf F}^{\rm gp}$, we optimize the parameters for GP kernels and noise with the Adam optimizer. The choice of kernels varies for different cases. For example, for the independent GP, we choose the RBF kernel matrix $K_{ij}$ for a given edge $e_{ij}$, defined as
\[
K_{ij}(e_{ij}) = \exp\left(-\frac{\| \mathbf{u}_i - \mathbf{u}_j \|^2}{2l_{e_{ij}}^2}\right),
\]
where the length scale $l$ and $\sigma$ are the trainable parameters, i.e., $\theta_{GP}=\{l, \sigma\}$. Specifically, in numerical experiments, we choose the same initialization on all edges. For the shared GP scenario, we consider the Automatic Relevance Determination (ARD) Matérn kernel \cite{Rasmussen2006} with $\nu=\tfrac{3}{2}$:
\[
K_{\nu=3/2}(x,x')
=
\sigma^2
\left(
1 + \sqrt{3} r
\right)
\exp(-\sqrt{3} r),
\quad
r^2 = \sum_d \frac{(x_d-x'_d)^2}{\ell_d^2}.
\]
Here, \(\ell_d>0\) is a trainable length scale associated with the \(d\)-th input feature, where a small \(\ell_d\) indicates that the predicted flux varies strongly in the \(d\)-th feature direction. This ARD parameterization allows the GP to have anisotropic scaling across different physical or geometric features. 

For the outer loop, with $\widehat{\mathbf F}^{\rm gp}$ and $\theta_{GP}$ fixed, we optimize over $\widehat{\mathbf u}$ and $\theta_W$ with the Shampoo optimizer \cite{gupta2018shampoo} by minimizing:
\begin{equation}\label{eq:opt_outer}
\begin{split}
	\mathcal{L}_{outer}&=\left\|\mathbf{F}_{\Gamma, \text{data}}
-R_\Gamma\left(\mathbf{M}_1^{-1}D^\top \widehat{\mathbf u} +\widehat{\mathbf F}^{\rm gp,\star}-\mathbf{M}_1^{-1}\mathbf{g}_D\right)\right\|^2\\
&\qquad + \left\|\mathbf{u}_{\text{data}}-\widehat{\mathbf u}^\top\Psi_0\right\|^2 \\
	\text{s.t.} &\quad D\mathbf{M}_1^{-1} D^\top \widehat{\mathbf u}
+D\widehat{\mathbf F}^{\rm gp,\star}
=\hat{\mathbf{b}}.\\
\end{split} \end{equation}
Essentially, we cast the optimal recovery problem on the reduced space subject to the equality constraint that enforces the conservation law exactly. Under this setting, we have a constrained optimization problem where the KKT conditions yield a saddle-point structure that we can use to derive a fast solution with the Schur complement. The bilevel training is not strictly necessary, but can be useful for efficient training when variables are coupled and different model components, here the GP and transformer, are trained jointly. 
\subsection{Geometric embedding for $H(\mathrm{div})$ coarse representations}\label{subsec:geom_embedding}
To incorporate geometric information about the coarse degrees of freedom into the GP, we introduce a geometric embedding based on domain integrals of the coarse Whitney basis functions. This embedding is designed to encode both the coarse variables in 0- and 1-forms in a manner consistent with the underlying FEEC structure. Because scalar $P_0$ coefficients represent coarse cellwise quantities while $H(\mathrm{div})$ coefficients represent flux degrees of freedom, we use different geometric summaries for 0-forms and 1-forms.

First, we consider the coarse $P_0$ geometric embedding 

\begin{proposition}
Let $\{\psi^0_i\}_{i=1}^{N_0}$ denote the coarse 0-form basis functions constructed as in \Cref{prop:coarse_0forms}.
For $i=1,\ldots,N_0$,  substituting the definition of $\psi^0_i$ gives the explicit representation for the geometric embedding on the $i$th PoU as
\begin{equation}
s_i
\coloneqq
\int_\Omega \psi_i^0(x)\,dx
=
\sum_{a\in\mathcal C_h}W_{ia}\int_\Omega \phi_a^{P_0}(x)\,dx.
\end{equation}
\end{proposition}
This embedding defines a positive coarse area scalar that represents the effective measure of the coarse region associated with the $i$-th PoU component and does not encode any directional information. In terms of the GP input, $s_i$ should be considered the geometric information for coarse node $i$.

We next define the coarse $RT_0$ geometric embedding that accounts for both direction and magnitude.
\begin{proposition}
Let $\{\psi^1_k\}_{k=1}^{N_1}$ denote the coarse $RT_0$ basis functions constructed as in \Cref{prop:coarse_1forms}. We define the vector-valued geometric embedding associated with the $k$-th coarse flux degree of freedom as
\[
s_k
\coloneqq
\int_{\Omega} \psi^1_k(x)\, dx
\;\in\;
\mathbb{R}^2.
\]

Substituting the expansion of $\psi^1_k$, we obtain:
\[
s_k
=
\sum_{e} W_{k K(e)}
\int_{\Omega} \varphi^{RT_0}_e(x)\, dx.
\]

For each fine edge $e$ shared by adjacent cells $K_L(e)$ and $K_R(e)$, the $RT_0$ basis function satisfies:
\[
\int_{K_L(e) \cup K_R(e)} \varphi^{RT_0}_e(x)\, dx
=
c_e\,\hat{n}_e\,|e|,
\]
where $\hat{n}_e$ is the unit normal vector associated with the edge orientation, $|e|$ is the edge length, and $c_e$ depends on the reference element geometry.
\end{proposition}

From the above expression, $\mathbf{s}_k$ is a weighted sum of edge-normal vectors $\hat{n}_e$ scaled by edge lengths $|e|$, and it encodes both the net orientation and the effective magnitude of the coarse flux degree of freedom $\psi^1_k$. Specifically, its direction reflects the dominant orientation of the fine-scale flux contributions aggregated into $\psi^1_k$, while its magnitude reflects the total weighted edge measure. This provides a geometric characterization of the coarse flux degree of freedom. Therefore, the above geometric embeddings are consistent with the structure of $H(\mathrm{div})$ spaces and enable the GP to incorporate geometric and directional information without relying on explicit spatial coordinates.

\subsection{Posterior error bound for boundary flux functionals}\label{subsec:error_bound}
The quantified uncertainty below is conditional on the learned representation from PoU Whitney forms and kernel hyperparameters. Under this conditioning, the boundary flux is a linear functional of the GP output at the coarse level and admits a closed-form posterior error bound. Previous work establishes a pointwise posterior error bound under the assumption that the GP input locations are fixed and known exactly \cite{propp2026discovery}. In this work, we extend this result to weighted linear functionals of the GP posterior, which is the relevant setting for boundary flux quantities.
\begin{proposition}\label{prop:F}
    Let $R_\Gamma(W)$ denote the coarse-to-fine reconstruction operator determined by $W$ in the proposed method \eqref{eq:opt_inner}. We define the flux functional on a given boundary $\Gamma$ and its estimator by
\begin{equation}
\begin{split}
      \mathcal{F}_\Gamma\coloneqq L_v(f)&\coloneqq v^\top f(Z)\coloneqq \sum_{m=1}^{M} v_m\, f(z_m),\\
\widehat{\mathcal{F}}_\Gamma&\coloneqq v^\top \hat f(Z),  
\end{split}
\end{equation}
where the weight $v^\top$ is the row vector of $R_\Gamma(W)$, i.e., $v_m=(R_\Gamma(W))_{m,:}$, and the error is defined by
\begin{equation}
    e\coloneqq\mathcal{F}_\Gamma-\widehat{\mathcal{F}}_\Gamma.
\end{equation}
\end{proposition}

\begin{theorem}\label{theorem:error}
Let $\mathcal X \subseteq \mathbb R^p$ be the Gaussian process input space, let $K:\mathcal X\times\mathcal X\to\mathbb R$ be a positive definite kernel
with associated RKHS $\mathcal H_K$, and let $f\in\mathcal{H}_K$. Let $X=(x_1,\ldots,x_N)\in\mathcal X^N,\,
Z=(z_1,\ldots,z_M)\in\mathcal X^M,$
where \(x_i,z_m\in\mathbb R^p\) and assume we have noisy data $\mathcal D=(X,Y)$ with $Y=f(X)+\varepsilon$ where $
\varepsilon\sim \mathcal N(0,\sigma_\varepsilon^2 I)$.
Let $\hat f\in\mathcal{H}_K$ denote the minimizer of the optimal recovery problem:
\[
\hat f(\cdot)=K(\cdot,X)\bigl(K(X,X)+\sigma_\varepsilon^2 I\bigr)^{-1}Y.
\]
Define the posterior covariance on $Z$ by 
$$ \Sigma_{Z}\coloneqq K(Z,Z)-K(Z,X)\Bigl(K(X,X)+\sigma_\varepsilon^2 I\Bigr)^{-1}K(X,Z).$$ 
For $v\in\mathbb R^M$, consider the weighted functionals $\mathcal F_\Gamma,\, \hat{\mathcal F}_\Gamma$ as defined in \Cref{prop:F}. Then we have the bounded mean-squared error satisfying:
\begin{equation}
\mathbb E_\varepsilon[e^2]\le
\|f\|_{\mathcal H_K}^2
v^\top\Sigma_Zv+\sigma_\varepsilon^2
\bigl\|v^\top K_{ZX}A^{-1}\bigr\|_2^2.
\label{eq:post-cov-bound}
\end{equation}
If we consider the diagonal bound, then \eqref{eq:post-cov-bound} can be written as
\begin{equation}
\mathbb E_\varepsilon[e^2]
\le
\|f\|_{\mathcal H_K}^2
\left(\sum_{m=1}^M |v_m|\,\sigma(z_m)\right)^2
+
\sigma_\varepsilon^2
\left\|\sum_{m=1}^M v_m K(z_m,X)(K_{XX} + \sigma_\varepsilon^2 I)^{-1}\right\|_2^2,
\label{eq:mse_error_bound}
\end{equation}
where $\sigma^2(z_m)\coloneqq(\Sigma_Z)_{mm}$.
\end{theorem}

\begin{proof}
Let $A=K(X,X)+\sigma_\varepsilon^2 I$. For finite ordered sets
\(Z=(z_i)_{i=1}^{M}\) and \(X=(x_j)_{j=1}^{N}\), define:
\[
K_{ZX}\coloneqq\bigl(K(z_i,x_j)\bigr)_{i,j}.
\]
Since $Y=f(X)+\varepsilon$,  we can decompose the error as
\begin{equation}
\begin{split}
  e :&= \mathcal{F}-\widehat{\mathcal{F}}\\
&= v^\top f(Z)-v^\top K_{ZX}A^{-1}Y\\
&=
\underbrace{v^\top f(Z)-v^\top K_{ZX}A^{-1}f(X)}_{=:\,e_{\text{det},v}(f)}
-
\underbrace{v^\top K_{ZX}A^{-1}\varepsilon}_{=:\,e_{\text{noise},v}}.\label{eq:e-def}  
\end{split}
\end{equation}
We then evaluate the MSE:
\begin{align}
\mathrm{MSE}(e)
&=
\mathbb{E}_\varepsilon[e^2]
=
\mathbb{E}_\varepsilon\bigl[(e_{\det}-e_{\text{noise}})^2\bigr]\\
&=
e_{\det}^2+\mathrm{Var}(e_{\text{noise}}).
\label{eq:MSE-split}
\end{align}
For the noise part, since $\varepsilon\sim\mathcal N(0,\sigma_\varepsilon^2I_N)$, we have:
\begin{equation}
    \mathbb E_\varepsilon[e_{\text{noise},v}]=0,
\qquad
\operatorname{Var}[e_{\text{noise},v}]=\sigma_\varepsilon^2\|v^\top K_{ZX}A^{-1}\|_2^2.
\end{equation}
By the reproducing property,
\begin{equation}
f(z_m)=\langle f,K(z_m,\cdot)\rangle_{\mathcal H_K},
\qquad
f(x_i)=\langle f,K(x_i,\cdot)\rangle_{\mathcal H_K}.    
\end{equation}
Therefore, we further write the deterministic error in \eqref{eq:e-def} as
\begin{equation}
    e_{\text{det},v}(f)=\left\langle
f,
\sum_{m=1}^M v_m
\Bigl[
K(z_m,\cdot)-K(z_m,X)A^{-1}K(X,\cdot)
\Bigr]
\right\rangle_{\mathcal H_K}
=
\langle f,r_v\rangle_{\mathcal H_K}.
\end{equation}
Thus 
\begin{equation}\label{eq:det}
    \mathbb{E}_\varepsilon[e^2]=\langle f,r_v\rangle_{\mathcal H_K}^2+\sigma_\varepsilon^2
\bigl\|v^\top K_{ZX}A^{-1}\bigr\|_2^2.
\end{equation}
By applying the Cauchy--Schwarz inequality, we have:
\begin{align}
|\langle f,r_v\rangle_{\mathcal H_K}|^2\le
\|f\|_{\mathcal H_K}^2\|r_v\|_{\mathcal H_K}^2.
\end{align}
We now compute $\|r_v\|_{\mathcal H_K}^2$. Recall the reproducing property that for ordered sets \(Z=(z_i)_{i=1}^{M}\) and \(X=(x_j)_{j=1}^{N}\), and vectors $a\in\mathbb{R}^M,\,b\in\mathbb{R}^N$,
\begin{align}
    \left\langle
a^\top K(Z,\cdot),
b^\top K(X,\cdot)
\right\rangle_{\mathcal H_K}&=\sum_{i=1}^{M}\sum_{j=1}^{N}
a_i b_j
\langle K(z_i,\cdot),K(x_j,\cdot)\rangle_{\mathcal H_K}\\
&=\sum_{i=1}^{M}\sum_{j=1}^{N}
a_i b_j K(z_i,x_j)=a^\top K_{ZX}b.
\label{eq:finite-reproducing-identity}
\end{align}
Then $r_v(\cdot)$ can be equivalently written as 
\begin{equation}
    r_v(\cdot)=v^\top K(Z,\cdot)-\alpha^\top K(X,\cdot),
\end{equation}
where
\begin{align*}
\alpha^\top
&=
v^\top K_{ZX}A^{-1}\\
K(Z,\cdot)
&=\bigl(K(z_1,\cdot),\ldots,K(z_M,\cdot)\bigr)^\top,\\
K(X,\cdot)
&=\bigl(K(x_1,\cdot),\ldots,K(x_N,\cdot)\bigr)^\top
\end{align*}
Applying \eqref{eq:finite-reproducing-identity} yields:
\begin{align}
    \|r_v\|_{\mathcal H_K}^2&=
\left\langle
v^\top K(Z,\cdot)-\alpha^\top K(X,\cdot),
v^\top K(Z,\cdot)-\alpha^\top K(X,\cdot)
\right\rangle_{\mathcal H_K}\\
&=
\left\langle v^\top K(Z,\cdot),v^\top K(Z,\cdot)\right\rangle_{\mathcal H_K}
-
2\left\langle v^\top K(Z,\cdot),\alpha^\top K(X,\cdot)\right\rangle_{\mathcal H_K}\\
&\qquad +
\left\langle \alpha^\top K(X,\cdot),\alpha^\top K(X,\cdot)\right\rangle_{\mathcal H_K}\\
&=v^\top K_{ZZ}v-2v^\top K_{ZX}A^{-1}K_{XZ}v+v^\top K_{ZX}A^{-1}K_{XX}A^{-1}K_{XZ}v\\
    &=v^\top K_{ZZ}v-v^\top K_{ZX}A^{-1}K_{XZ}v-\sigma_\varepsilon^2v^\top K_{ZX}A^{-2}K_{XZ}v\\
    &=v^\top\Sigma_Zv-\sigma_\varepsilon^2
v^\top K_{ZX}A^{-2}K_{XZ}v\\
&=v^\top\Sigma_Zv-\sigma_\varepsilon^2\left\|A^{-1}K_{XZ}v\right\|_2^2\le v^\top\Sigma_Zv\label{eq:det-cauchy}
\end{align}
since $\left\|A^{-1}K_{XZ}v\right\|_2^2\ge 0$ for $K_{XX}\ge 0$ and symmetric positive definite $A$.

Hence with \eqref{eq:det} and \eqref{eq:det-cauchy}, we obtain:
\begin{equation}
\mathbb E_\varepsilon[e^2]\le
\|f\|_{\mathcal H_K}^2
v^\top\Sigma_Zv+\sigma_\varepsilon^2
\bigl\|v^\top K_{ZX}A^{-1}\bigr\|_2^2.
\end{equation}
The decomposition of error into noise and deterministic components is readily interpretable. The deterministic component captures how well the training data at $X$ could reconstruct the functional $v^\top f(Z)$ if the observations were noiseless. This error is largely controlled by the RKHS norm (or complexity) of the true function $f$ and by the posterior covariance at evaluation points $Z$. The noise component, controlled by $\sigma_{\varepsilon}^2\|v^\top K_{ZX}A^{-1}\|_2^2$, captures how much the observation noise is amplified through the kernel reconstruction and through the linear functional $v^\top$.

Finally, we derive the diagonal posterior standard deviation from the full covariance form. Recall that $\Sigma_{Z}$ is a positive semidefinite block matrix. We therefore have \begin{equation}
    |(\Sigma_Z)_{mn}|
\le
\sqrt{(\Sigma_Z)_{mm}}\sqrt{(\Sigma_Z)_{nn}}
=
\sigma(z_m)\sigma(z_n).
\end{equation}
Thus, we have:
\begin{equation}\begin{aligned}\label{eq:post_std}
    v^\top\Sigma_Zv&=\sum_{m,n=1}^M v_mv_n(\Sigma_Z)_{mn}\\
    &\le\sum_{m,n=1}^M|v_m||v_n|\sigma(z_m)\sigma(z_n)\\
    &=\left(\sum_{m=1}^M |v_m|\sigma(z_m)\right)^2
\end{aligned}\end{equation}
Substituting \eqref{eq:post_std} into \eqref{eq:post-cov-bound} gives the desired error bound in the diagonal form \Cref{eq:mse_error_bound}. This bound can be interpreted as yielding large errors if the functional places a large weight on uncertain locations.
\end{proof}

\section{Results}
\label{sec:results}
In this section, we demonstrate the performance of our proposed method on various problems, including the 1D Poisson equation with analytic solution (\Cref{subsec:1d_linear}), the nonlinear diffusion equation with manufactured solution (\Cref{subsec:1d_nonlinear}), and the 2D advection-diffusion equation on a bell-shaped complex geometry (\Cref{subsec:bell}). We also present results for a semiconductor $p-n$ diode problem (\Cref{subsec:diode}) to illustrate the range over which the uncertainty estimates remain reliable for the proposed method. For the training of data-driven Whitney forms, we consider a lightweight cross-attention transformer with dropout.

\subsection{Toy one-dimensional linear example}
\label{subsec:1d_linear}
Consider a toy example of the one-dimensional Poisson equation on $\Omega = (0,1)$:
\begin{equation}\label{eq:1d_poisson}
  F = -\nabla u, \qquad
  \nabla \cdot F = f,
\end{equation}
with Dirichlet boundary conditions $
  u(0) = \alpha$, $u(1) = 0$, and constant source term $f \equiv 1$.  The analytic solution is:
\begin{equation}
  u(x) = -\tfrac{x^2}{2} + \bigl(\tfrac{1}{2} - \alpha\bigr)x + \alpha,
  \qquad
  F(x) = x + \alpha - \tfrac{1}{2}. 
\end{equation}
The goal is to learn the solution $u$ and the associated flux $F$ given the left boundary condition $\alpha$. In particular, we use this toy example because the flux depends linearly on the left boundary condition. 

In the implementation, to construct the fine-scale discretization, we employ the classic Galerkin finite element with $P_0$ elements for solution $u$ and $P_1$ Lagrange elements for the flux. For a mesh with $N$ cells, we have $N_{cell}=N$ cell-averaged values for $u$ and $N_{face}=N+1$ nodal flux values for $F$. Then the sparse basis evaluations in the fine scale $\phi^0\in\mathbb{R}^{N_{\text{cell}}\times N_q}$ and $\phi^1\in\mathbb{R}^{N_{\text{face}}\times N_q}$ are constructed with the quadrature weight exponent $\frac{1}{2}$. In one dimension, this is analogous to the lowest-order Raviart--Thomas space $\mathrm{RT}_0$ where each hat function $\phi_i^1$ is supported on the two cells adjacent to node $i$, and its derivative $d\phi_i^1/dx$ plays the role of $\nabla\!\cdot\!\phi_i^F$. The resulting $P_0$--$P_1$ pair yields the standard mixed $P_0$--$H(\mathrm{div})$ formulation in 1D. In addition, for a prescribed number of PoU, we have $N_0=N_{\text{PoU}}$  for the coarse 0-forms and $N_1^{\text{int}}=\binom{N_0}{2}$ for the interior coarse 1-forms, respectively. For the 1D problem, we explicitly set two boundary 1-forms $\psi_1^{\rm left}$ and $\psi_1^{\rm right}$, which gives the $N_1^{bc}=2$ and total number of coarse 1-forms $N_1=\binom{N_0}{2}+2$.

Comparison of the ground truth with the prediction is shown in Fig.~\ref{fig:1d_toy_post} (left). The reconstructed solution field from the coarse 0-forms can approximate the ground-truth solution with the MSE below $10^{-5}$. Fig.~\ref{fig:1d_toy_post} (right) illustrates the GP inference of the left boundary flux with respect to different values of $\alpha$. For the 1D problem where the target boundary $\Gamma$ is reduced to the left boundary point $x=0$, the coarse-to-fine reconstruction of flux in~\Cref{prop:F} is simplified to use the coarse boundary 1-form for the left boundary. Furthermore, as this problem is linear, the error bound is still tight near the training range even in the small extrapolation area, with only a slight increase in the most extreme cases, $\alpha=0.5$ and $\alpha=3.5$.
\begin{figure}[t]
    \centering
    \begin{subfigure}{0.48\textwidth}
        \centering
        \includegraphics[width=\linewidth]{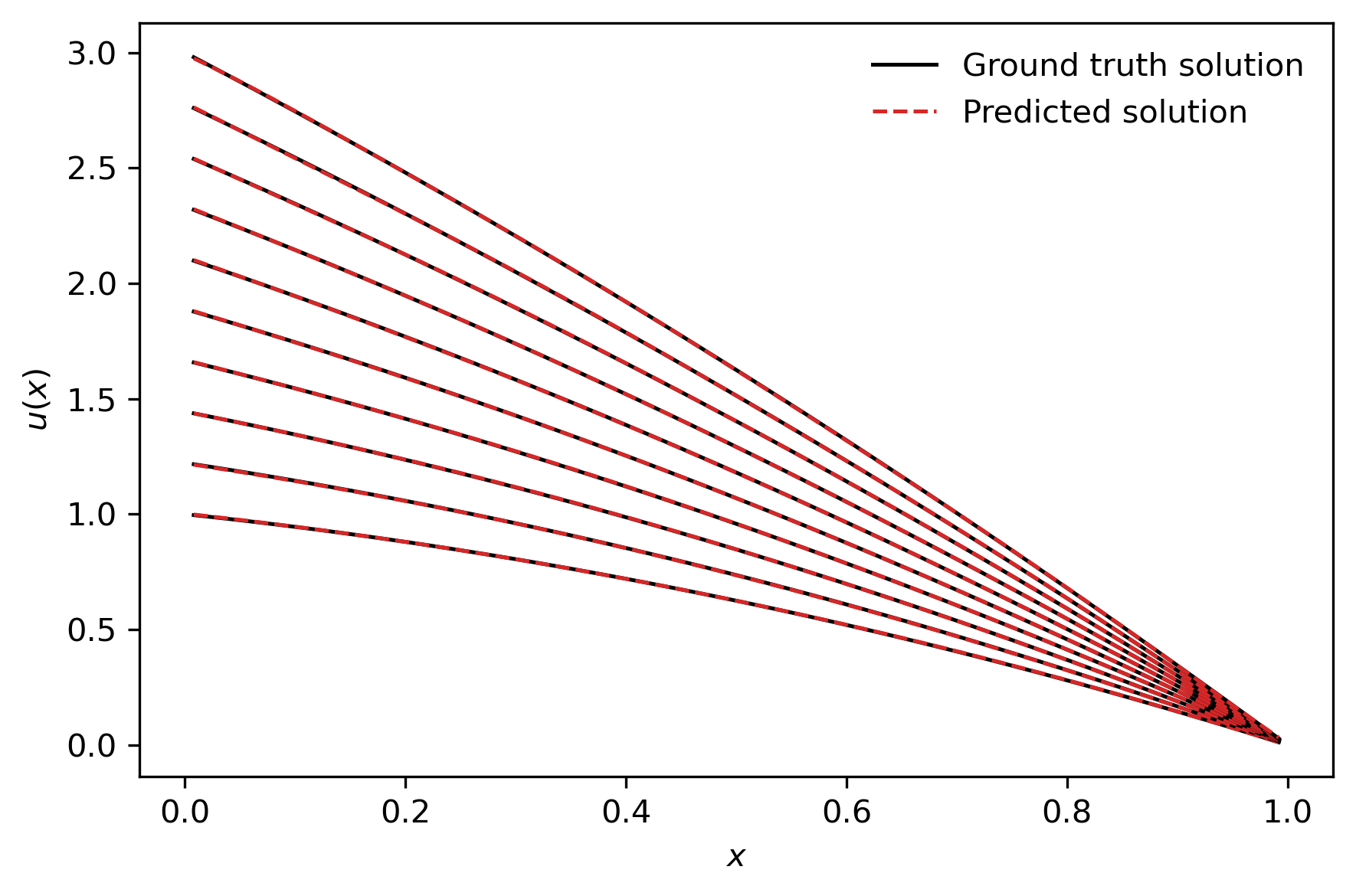}
    \end{subfigure}
    \hfill
    \begin{subfigure}{0.48\textwidth}
        \centering
        \includegraphics[width=\linewidth]{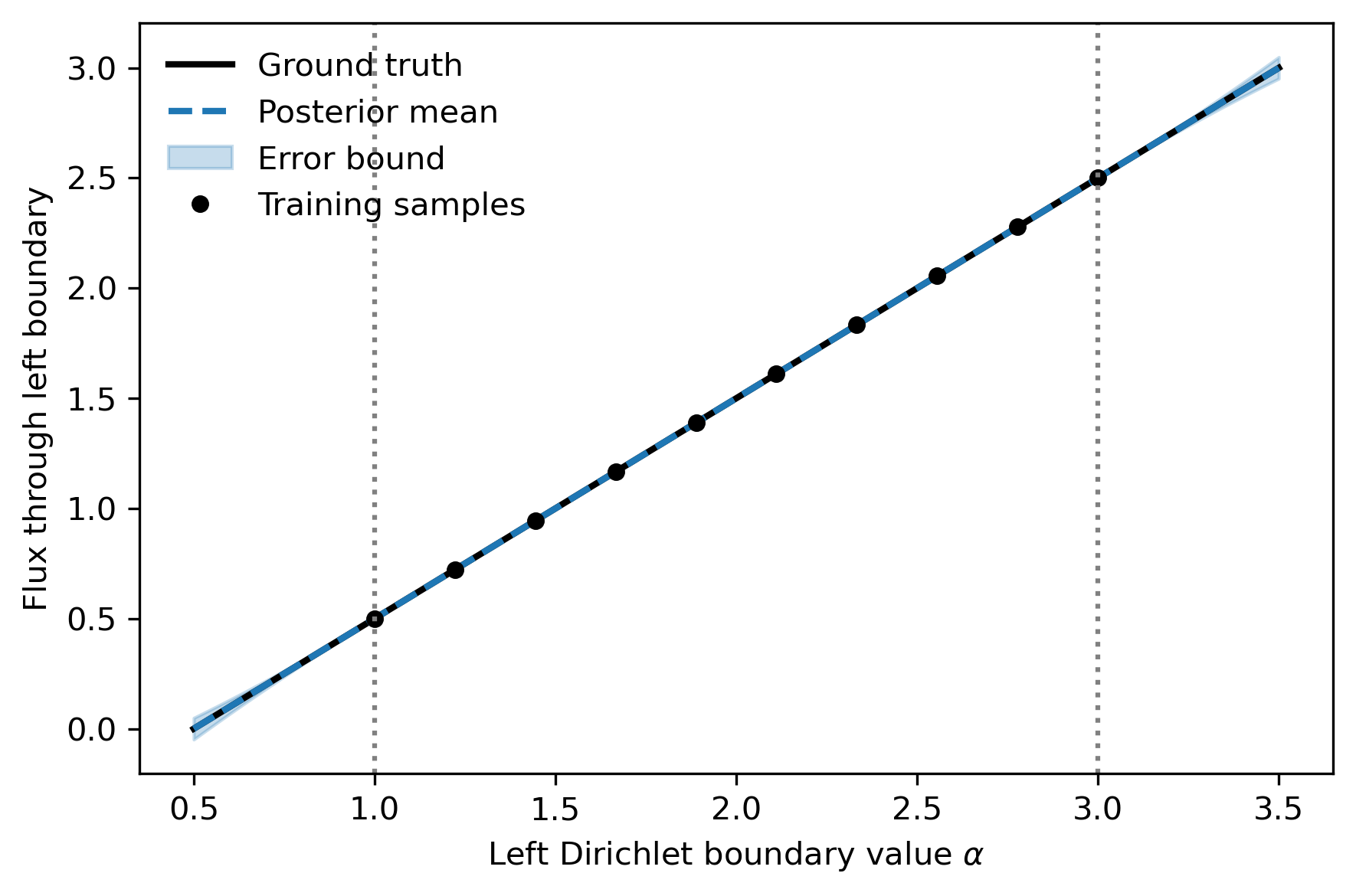}
    \end{subfigure}
    \caption{\textbf{Solution reconstruction and quantified uncertainty for the toy example.}  \textbf{Left}: Solution field $u$ with respect to different left boundary conditions $\alpha$. The black solid line is the analytic solution, and the red dashed line is the reconstructed solution. \textbf{Right}: Inference result for the left boundary flux with respect to left boundary conditions $\alpha$ with shaded blue area as the error bound from Theorem~\ref{theorem:error}.}
    \label{fig:1d_toy_post}
\end{figure}

\subsection{1D nonlinear example with manufactured solution}
\label{subsec:1d_nonlinear}
We next consider a nonlinear diffusion equation with piecewise nonlinear diffusivity and the same Dirichlet boundary conditions as in \Cref{subsec:1d_linear}:
\begin{equation}\label{eq:1d_nonlinear_mixed}
  F = -k(u)\,\nabla u, \qquad
  \nabla \cdot F = f,
\end{equation}
where
\begin{equation}\label{eq:k_of_u}
  k(u) =
  \begin{cases}
    k_0 & \text{if } u \le u_0, \\
    \beta\,(u - u_0)^q + k_0 & \text{if } u > u_0.
  \end{cases}
\end{equation}
The analytic solution is $u(x) = \alpha(1 - x)$, which gives $F(x) = \alpha\, k\bigl(\alpha(1-x)\bigr)$ and 
\[
  f(x) = \nabla \cdot F(x) =
  \begin{cases}
    0 & \text{if } \alpha(1-x) \le u_0, \\
    -\beta q\alpha^2 \bigl(\alpha(1-x)-u_0\bigr)^{q-1}& \text{if } \alpha(1-x) > u_0.
  \end{cases}
\]
This problem features a nonlinear constitutive law $k(u)$, a
spatially varying source $f(x)$ that depends on the parameter $\alpha$,
and a kink in $k$ at the threshold $u = u_0$. Therefore, we have a manufactured problem where the flux is a
nonlinear function of the left boundary condition $\alpha$ with a transition point at $x=1-\frac{u_0}{\alpha}$.
\begin{figure}[htbp]
    \centering   \includegraphics[width=\linewidth]{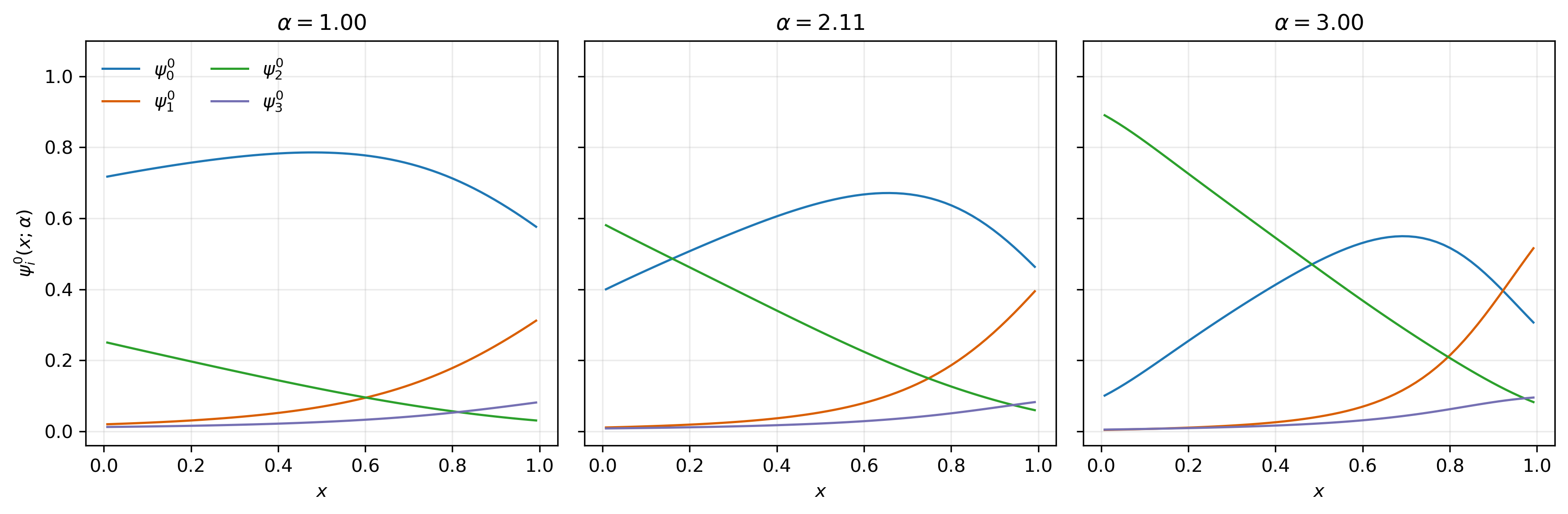}
    \caption{\textbf{Conditional coarse bases for \Cref{subsec:1d_nonlinear}}. Coarse shape functions $\{\psi^0_i\}$ evaluated on the fine-scale nodes can adapt to different values $\alpha=[1,2.1, 3]$.}
    \label{fig:1d_nonlinear_pou}
\end{figure}
\begin{figure}[htbp]
    \centering
\includegraphics[width=0.6\linewidth]{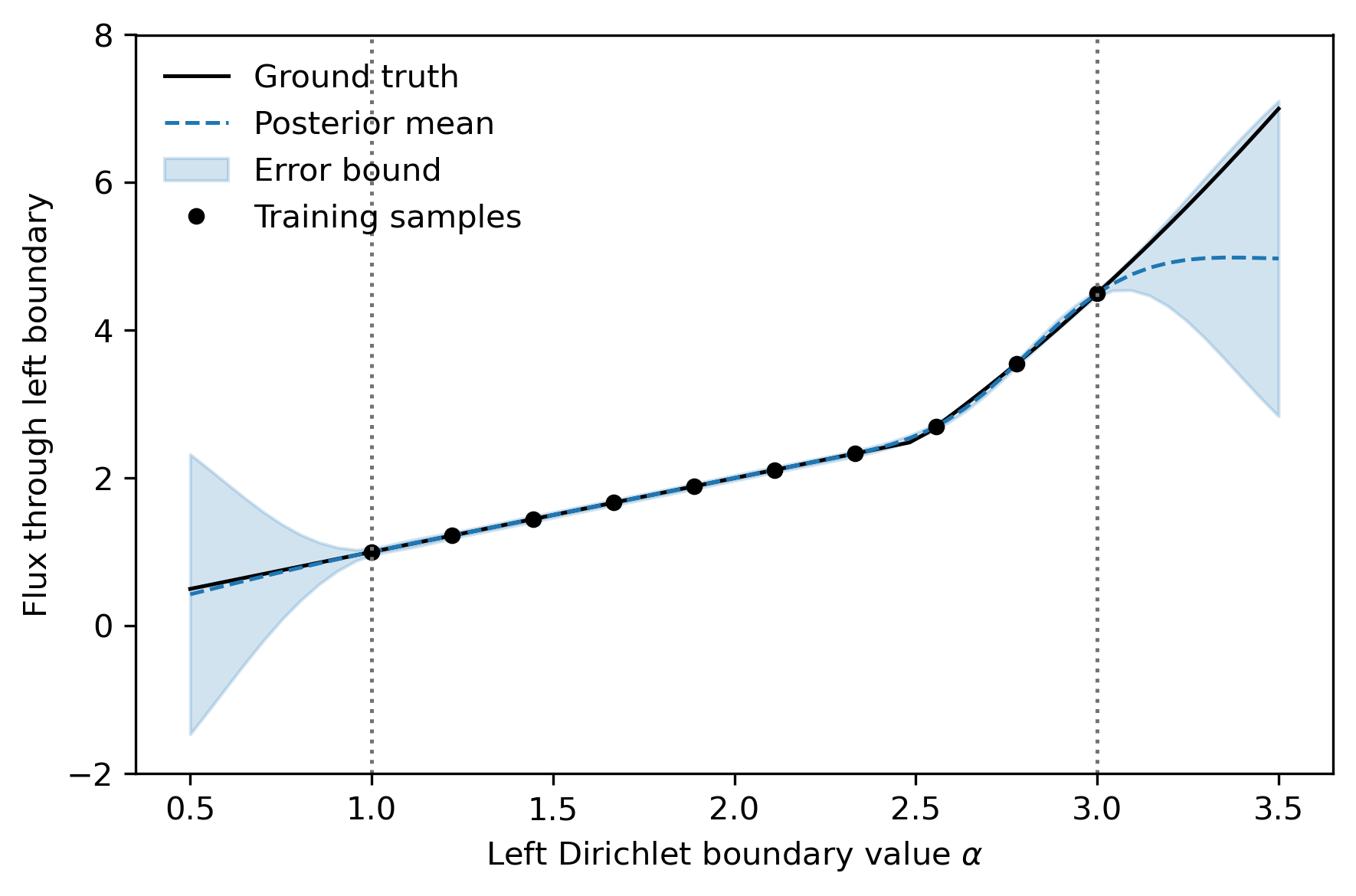}
    \caption{\textbf{Quantified uncertainty for the left-boundary flux in the 1D nonlinear problem.} The learned surrogate predicts the boundary flux \(F_\Gamma\) as a function of the left Dirichlet boundary value \(\alpha\). The growth of the uncertainty band outside the training range in $\alpha$ reflects reduced confidence in extrapolation.} 
    \label{fig:1d_nonlinear_post}
\end{figure}
The visualization of coarse shape functions $\{\psi_j^0\}$ in Fig.~\ref{fig:1d_nonlinear_pou} demonstrates that the conditional neural Whitney form allows an adaptive adjustment of the PoU. As $\alpha$ changes, the coarse basis functions reallocate their support to better accommodate the condition. One coarse basis function becomes much more dominant near the left side at larger values of $\alpha$, suggesting that the solution structure there becomes more parameter-sensitive. For this example with both linear and nonlinear regions, the error bound grows more noticeably in the extrapolation range than in the previous linear example (Fig.~\ref{fig:1d_nonlinear_post}). Furthermore, in the left extrapolation region where the flux is still linear in the conditioning variable, the GP-inferred flux deviates less from the ground-truth than in the nonlinear extrapolation region on the right. The error bound grows in both extrapolation directions because the distance in the GP feature space from the training samples increases regardless of whether the extrapolation is in the linear regime or the nonlinear regime. 

\subsection{Conditioned 2D advection-diffusion equation in complex geometry}
\label{subsec:bell}
We extend the problem to the steady advection-diffusion equation in a complex geometry and an unstructured mesh:
\begin{equation}\label{eq:2d_advdiff}
  -\nabla\cdot(k\,\nabla u) + \boldsymbol{\beta}\cdot\nabla u = 0
  \quad \text{in } \Omega,
\end{equation}
where $k$ is the diffusion coefficient and $\boldsymbol{\beta} = (\cos\theta, \sin\theta)$ is the advection direction with $\theta \in [0, 2\pi)$. The domain is a two-dimensional ``bell''-shaped domain $\Omega$ with three distinct Dirichlet
boundary conditions as shown in Fig.~\ref{fig:2d_mesh}. We use an unstructured triangular mesh to discretize the complex bell geometry and parameterize the problem with both $\theta$ and $u_{crack}$. Specifically, we have the top handle of the bell ($\Gamma_1$) with $u = 1$, the crack boundary ($\Gamma_2$) with $u = u_{crack}$, and the remaining bell-body boundary ($\Gamma_3$) with $u = 0$. The conditioning variable $\theta$ controls the advection direction and thus changes the transport pattern inside the domain.
\begin{figure}[htbp] \centering\includegraphics[width=\linewidth]{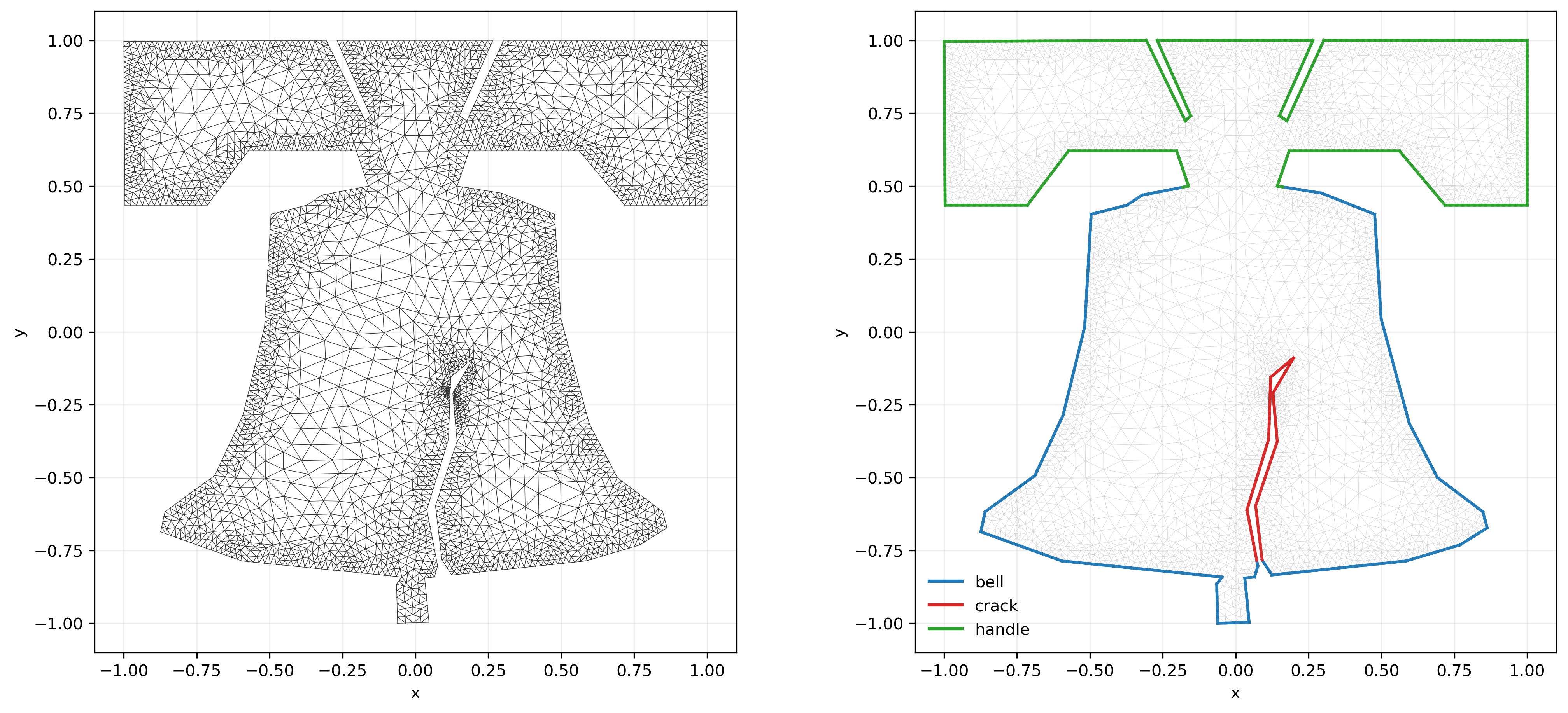}
    \caption{\textbf{Complex geometry demonstration.} Left figure shows the irregular triangular mesh and right figure shows the nonhomogeneous Dirichlet boundary conditions.}
    \label{fig:2d_mesh}
\end{figure}

\begin{figure}[t]
    \centering
    \begin{subfigure}{0.9\textwidth}
    \centering
    \includegraphics[width=\linewidth]{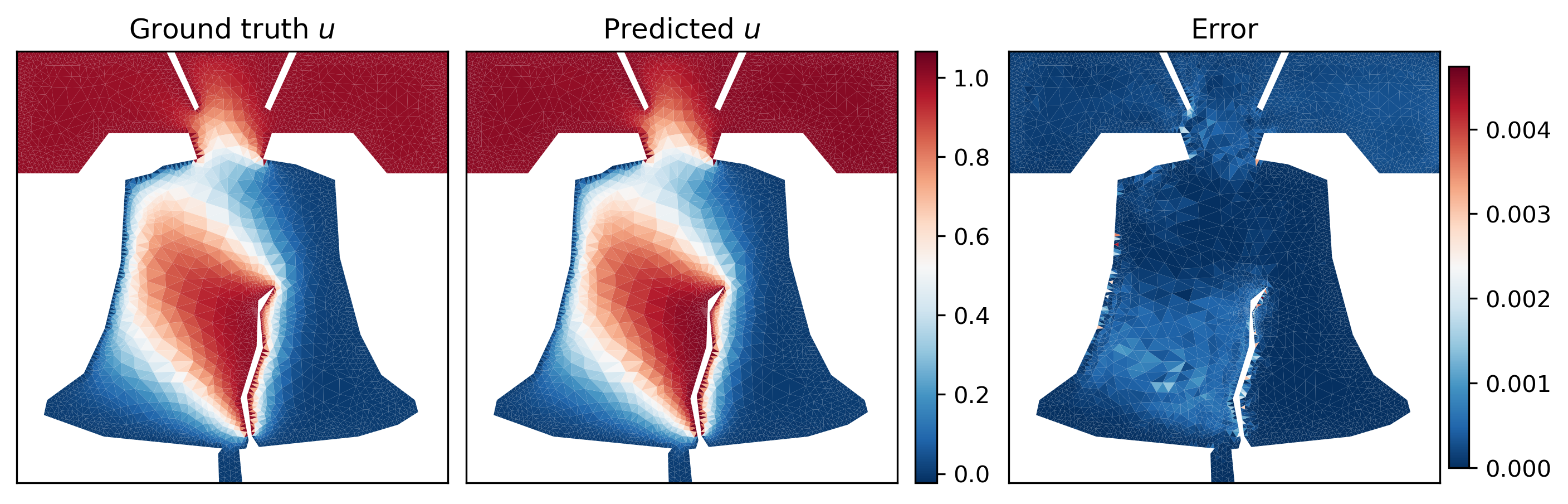}
    \label{fig:2d_solution}  
    \end{subfigure}
    \begin{subfigure}{0.48\textwidth}
        \centering
        \includegraphics[width=\linewidth]{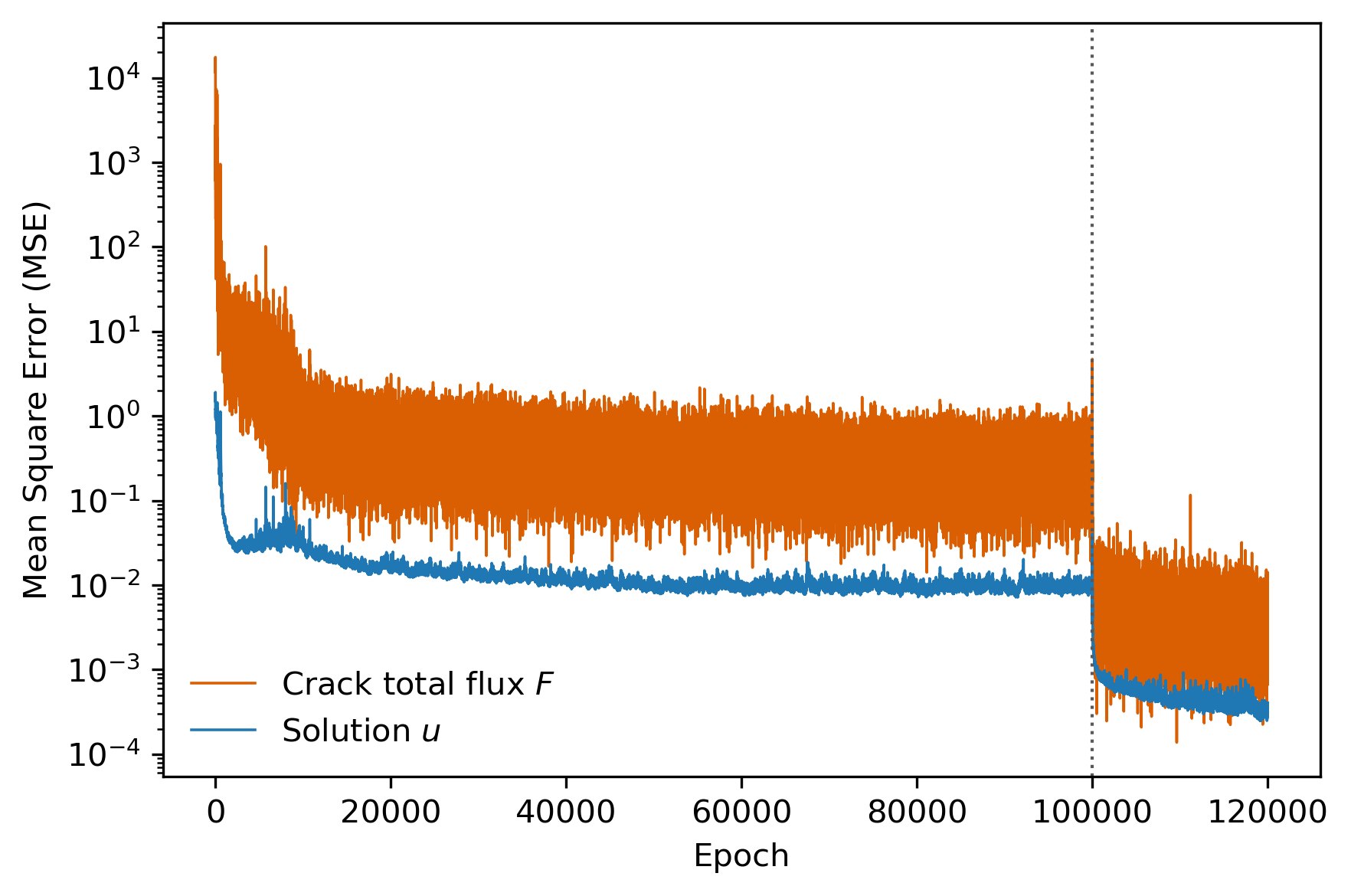}
    \label{fig:2d_loss}
    \end{subfigure}
    \hfill
    \begin{subfigure}{0.48\textwidth}
        \centering
        \includegraphics[width=\linewidth]{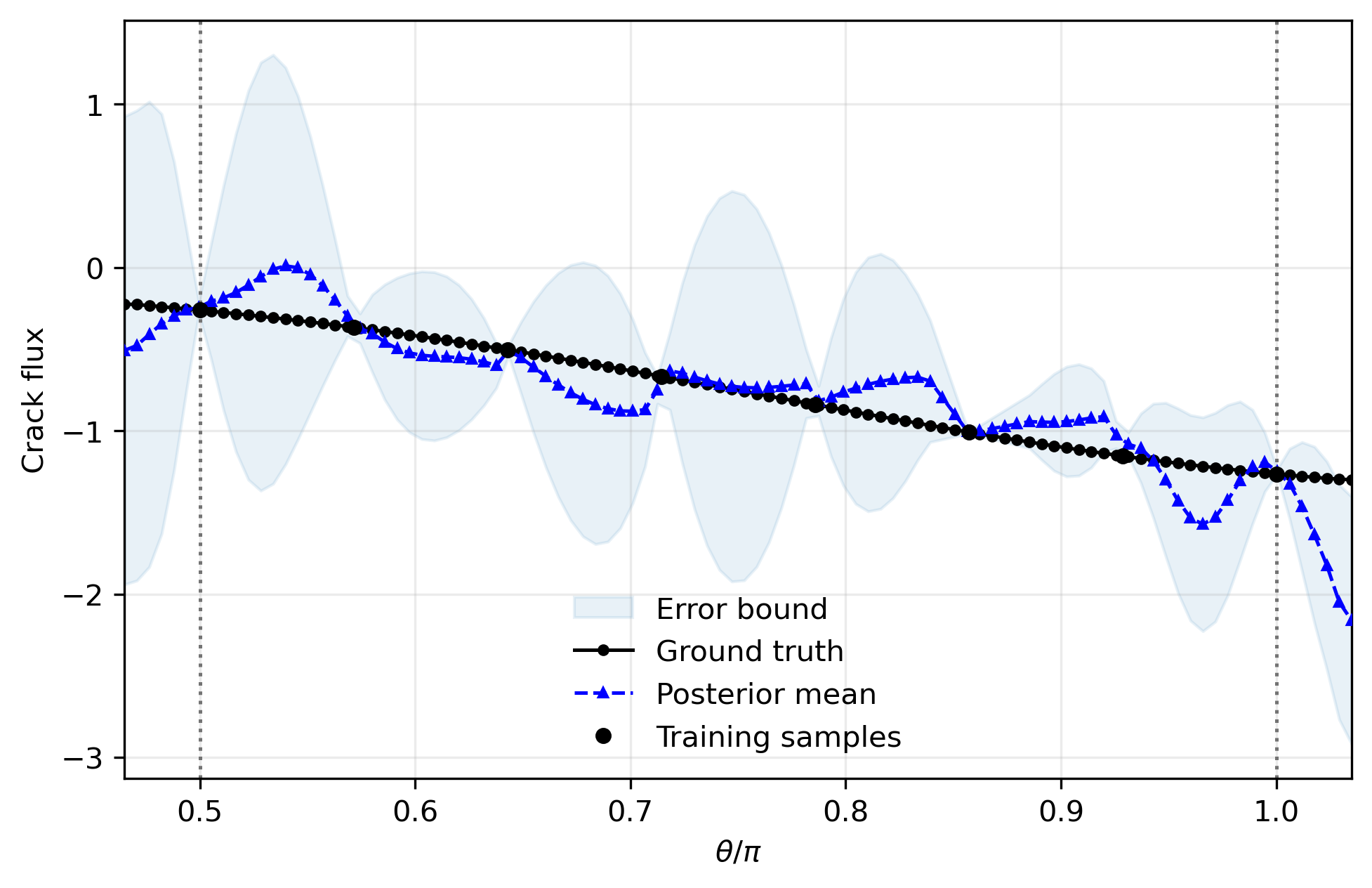}
\label{fig:2d_post}
    \end{subfigure}
    \caption{\textbf{Quantified uncertainty on complex geometry.} \textbf{Top}: The solution to the advection-diffusion equation, conditioned on the advection direction, is recovered. \textbf{Bottom:} Dropout during training provides a high-quality basis; turning off at 100k epochs allows refinement of the final basis (\textit{bottom left}). The error in the flux through the crack at the bottom of the Liberty Bell is bounded by the estimator in \Cref{eq:mse_error_bound} (\textit{bottom right}). }\label{fig:2d_result}
\end{figure}
Results are shown in Fig.~\ref{fig:2d_result}. In the implementation, we first train the transformer for the Whitney forms with a dropout rate of 0.1 for 100,000 epochs and then turn off the dropout for another 20,000 epochs. Empirically, this two-stage training strategy alleviates overfitting and prevents the transformer from getting stuck in local minima during training. Fig.~\ref{fig:2d_pou} visualizes the Whitney 0-forms learned by the neural PoU framework for three different values of the conditioning variables, showing that the learned partitions adapt naturally to the geometric structure. As $\theta$ and $u_{crack}$ increase, some 0-forms tend to localize near the crack interface to capture changes near the crack boundary efficiently, while others tend to align with the flow near other boundaries or within the domain. 
\begin{figure}[htbp]
    \centering
    \begin{subfigure}{0.9\textwidth}
    \centering
    \includegraphics[width=\linewidth]{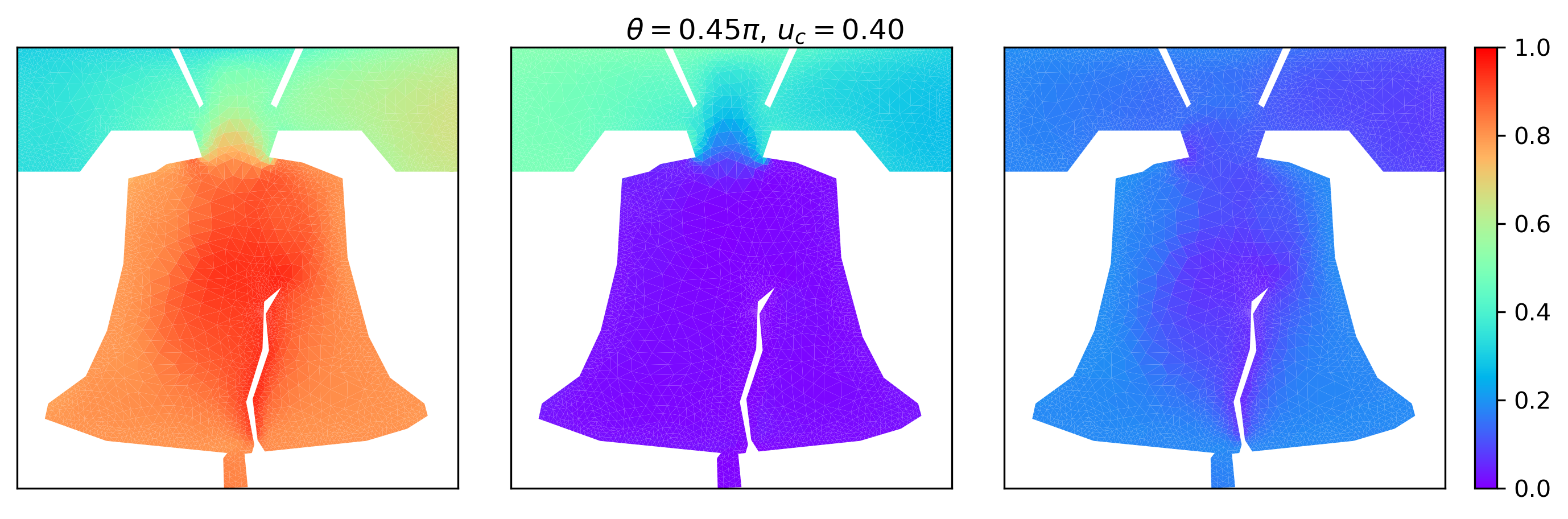}
    \label{fig:2d_ex1}  
    \end{subfigure}
    \begin{subfigure}{0.9\textwidth}
        \centering
        \includegraphics[width=\linewidth]{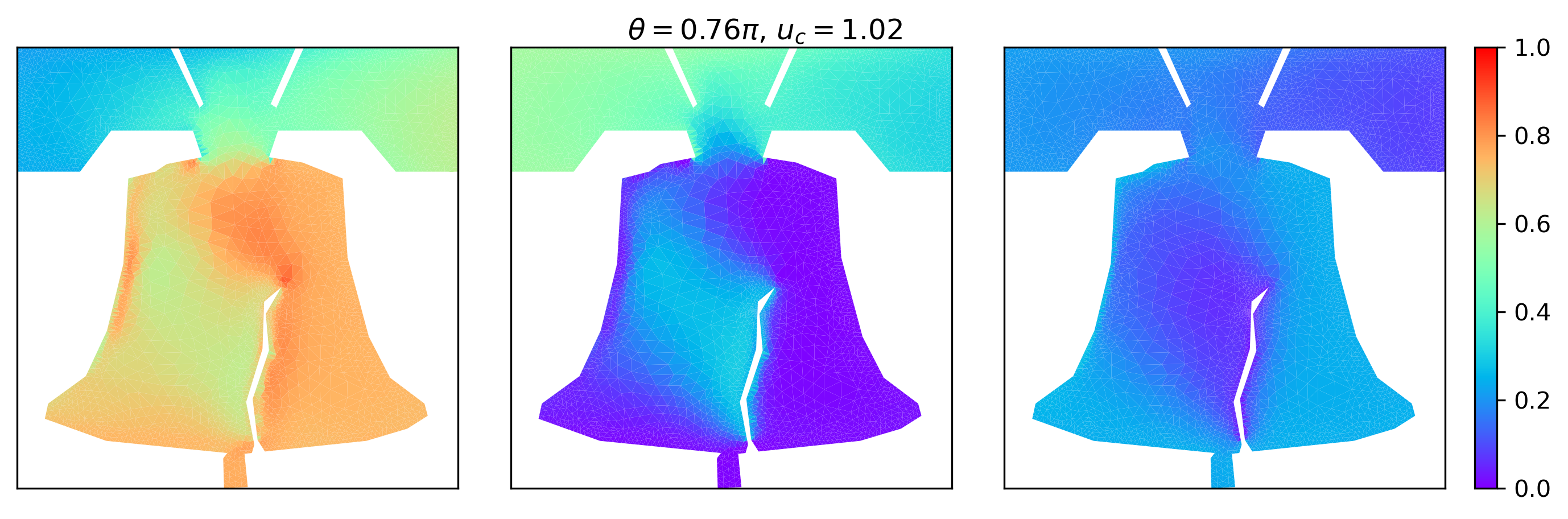}
    \label{fig:2d_ex2}
    \end{subfigure}
    \hfill
    \begin{subfigure}{0.9\textwidth}
        \centering
        \includegraphics[width=\linewidth]{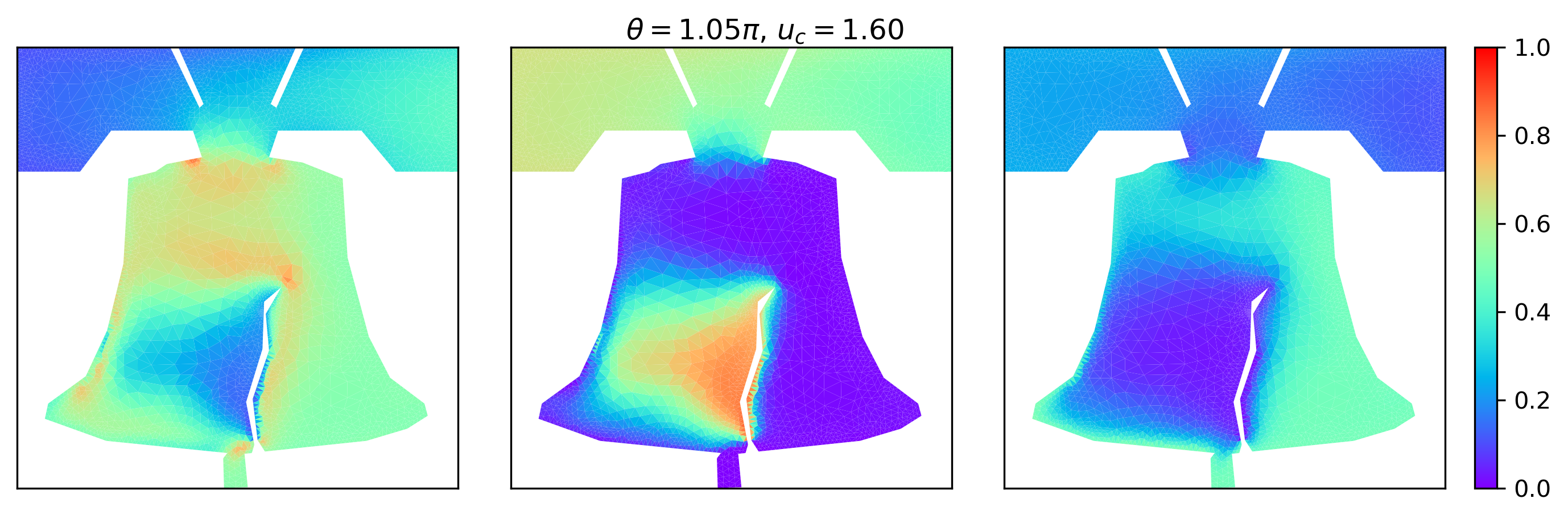}
    \label{fig:2d_ex3}
    \end{subfigure}
    \caption{\textbf{Conditioned basis adapts to advection.} The three basis functions are shown conditioned on angles $\theta \in [0.45\pi, 0.75\pi, 1.05\pi]$, respectively, highlighting the ability of the basis to adapt itself to the relevant conditioning. }
    \label{fig:2d_pou}
\end{figure}

\subsection{$p-n$ diodes}
\label{subsec:diode}
Finally, we consider the $p-n$ diode problem obtained from Charon TCAD simulations~\cite{osti_1575982} on the rectangular device domain $\Omega=(0,L_x)\times (0,L_y)$ with $L_x=1$ and $L_y=0.5$. For each applied anode bias \(V_a\), the data consist of the electrostatic potential, electron and hole current densities, doping profiles, and the terminal currents at the cathode and anode.  The cathode bias is fixed and the anode bias is swept over \([-1,1]\) in thermal-voltage units; multiplication by \(V_0=2.585\times 10^{-2}\,\mathrm{V}\) gives the corresponding physical voltage scale. 

The two-dimensional TCAD fields are reduced to one-dimensional profiles in a way that preserves the terminal-current constraint.  Let \(\varphi(x,y; V_a)\) denote the electrostatic potential and let \(j_x(x,y;V_a)\) denote the net conventional current density in the axial direction.  On the structured \(100\times 10\) cell grid, we define:
\begin{equation}
  u(x;V_a)=\frac{1}{L_y}\int_0^{L_y}\varphi(x,y;V_a)\,dy,
  \qquad
  J(x;V_a)=\int_0^{L_y}j_x(x,y;V_a)\,dy .
  \label{eq:diode_1d_reduction}
\end{equation}
Since the steady drift--diffusion solution satisfies charge conservation (\(\nabla\cdot j=0\)) and the top and bottom boundaries are insulating, integration over the cross section gives:
\begin{equation}
  \frac{d}{dx}J(x;V_a)=0,
  \qquad
  J(x;V_a)\simeq I_a(V_a).
  \label{eq:diode_current_conservation}
\end{equation}

\begin{figure}[htbp]
    \centering
    \begin{subfigure}[t]{0.4\linewidth}
\centering\includegraphics[width=\linewidth]{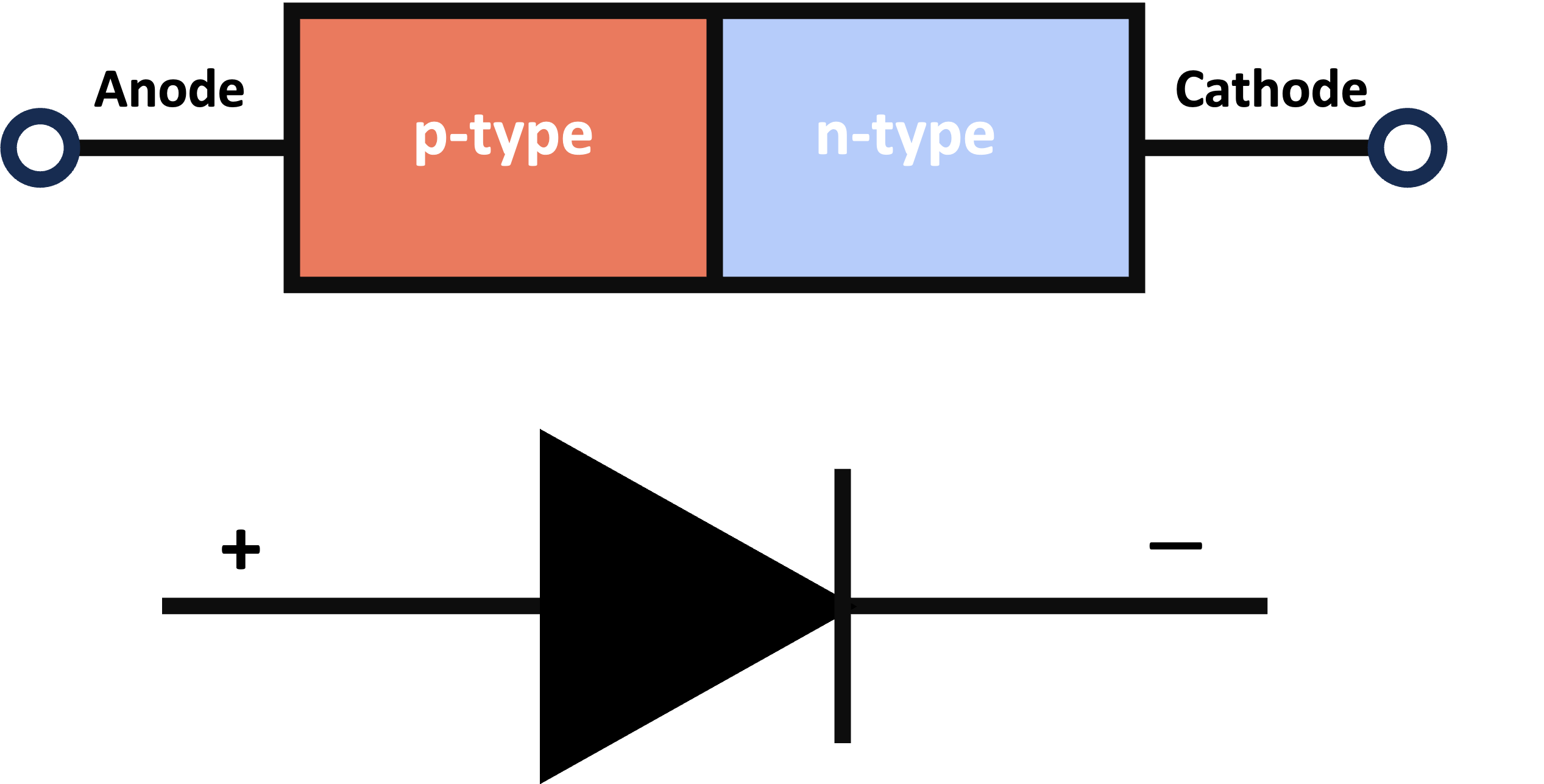}
    \end{subfigure}
    \vspace{0.5em}     
    \begin{subfigure}[c]{\linewidth}
    \includegraphics[width=\linewidth]{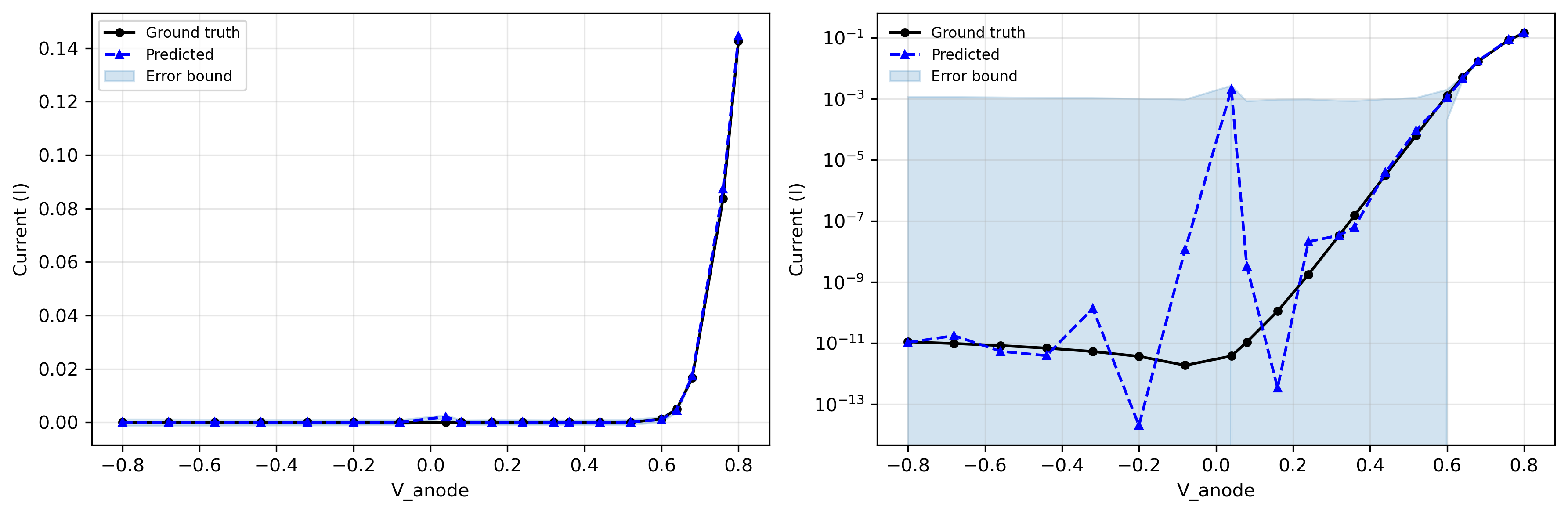}
    \end{subfigure}
    \caption{\textbf{Quantified uncertainty for the I-V response of a semiconductor component.} \textbf{Top:} schematic of the one-dimensional p-n diode and the prescribed anode--cathode orientation. \textbf{Bottom:} Comparison of the ground-truth current and predicted current in original scale (\textit{bottom left}) and log scale (\textit{bottom right}). Despite the appearance of a relatively good current reconstruction on a linear scale (\textit{bottom left}), on a log scale the uncertainty estimator reveals the range of voltages over which the model may be reliably applied.}
    \label{fig:1d_diode}
\end{figure}

We use the same conditional Whitney--GP construction as in previous 1D examples, where the neural Whitney forms are generated by the transformer conditioned on $V_a$. Given that the current values vary over several orders of magnitude, we train the GP on the logarithm of the current magnitude and map the posterior mean back to the original scale by exponentiation, i.e., the posterior mean $\mu(V_a)$ is mapped back to the original current magnitude by $\exp(\mu(V_a))$. The numerical results are shown in Fig.~\ref{fig:1d_diode} with the inference in original scale on the left and in log scale on the right. Despite the fairly good current reconstruction on a linear scale, the uncertainty estimator reveals the range of voltages over which the model may be reliably applied, and flags where the uncertainty becomes large relative to the scale of the quantity of interest. This example should be considered an illustration and guideline for where the proposed model starts to become less trustworthy, i.e., the posterior uncertainty grows rapidly in the exponential regime.
\section{Conclusion}\label{sec:conclusion}
In this work, we developed a structure-preserving neural surrogate framework for PDE-governed systems that incorporates a reduced finite element space and tractable uncertainty quantification with a closed-form posterior error bound. Our method separates the problem into two coupled parts: in \textbf{P1}, the trainable PoU Whitney forms generate the reduced-order $H(\mathrm{div})$--$L^2$ spaces that preserve the conservation law with the induced coarse divergence operator and support the graph interpretation; in \textbf{P2}, the GP learns the nonlinear state-to-flux map on the coarse graph by solving the optimal recovery problem with an equality constraint that enforces the conservation exactly. In particular, the conditional Whitney forms provide an adaptive mechanism that enables generalization across different geometries, and the optimal recovery formulation yields a saddle-point KKT condition with a fast Schur-complement solve. Numerical results demonstrate that the proposed method can reconstruct accurate solution fields and fluxes and provide meaningful uncertainty quantification for boundary flux functionals. In future work, we will generalize the framework to more complex problems and real datasets for practical applications, such as large-scale deployments for climate modeling. In addition, the current structure-preserving construction is based on the de Rham complex and is primarily designed for scalar- and vector-valued problems;  extending the method to handle tensor-valued variables is an important direction for future work.
\section*{Acknowledgments}
H. Zhang and Dr. Kinch acknowledge support from the United States Department of Energy under the Advanced Scientific Computing Research (ASCR) program (award number DE-SC0024563). Drs. Kinch, Owhadi, Propp, and Trask acknowledge funding under the ASCR-sponsored Mathematical Multifaceted Integrated Capability Centers program (award number DE-SC0023163).

\bibliographystyle{siamplain}
\bibliography{references}

@Book{boffi2013mixed,
  author =    {Daniele Boffi and Franco Brezzi and Michel Fortin},
  title =     {Mixed Finite Element Methods and Applications},
  series =    {Springer Series in Computational Mathematics},
  volume =    {44},
  publisher = {Springer},
  address =   {Berlin, Heidelberg},
  year =      {2013},
  doi =       {10.1007/978-3-642-36519-5},
  isbn =      {978-3-642-36518-8}
}

@article{propp2026discovery,
  title={Discovery of Probabilistic Dirichlet-to-Neumann Maps on Graphs},
  author={Propp, Adrienne M and Actor, Jonas A and Walker, Elise and Owhadi, Houman and Trask, Nathaniel and Tartakovsky, Daniel M},
  journal={SIAM Journal on Scientific Computing},
  volume={48},
  number={2},
  pages={C191--C215},
  year={2026},
  publisher={SIAM}
}

@book{Rasmussen2006,
  title     = {Gaussian Processes for Machine Learning},
  author    = {Rasmussen, Carl Edward and Williams, Christopher K. I.},
  year      = {2006},
  publisher = {The MIT Press},
  address   = {Cambridge, Mass.},
  isbn      = {026218253X},
  url       = {http://www.gaussianprocess.org/gpml/}
}

@inproceedings{gupta2018shampoo,
  title={Shampoo: Preconditioned stochastic tensor optimization},
  author={Gupta, Vineet and Koren, Tomer and Singer, Yoram},
  booktitle={International Conference on Machine Learning},
  pages={1842--1850},
  year={2018},
  organization={PMLR}
}

@article{actor2024data,
  title={Data-driven Whitney forms for structure-preserving control volume analysis},
  author={Actor, Jonas A and Hu, Xiaozhe and Huang, Andy and Roberts, Scott A and Trask, Nathaniel},
  journal={Journal of Computational Physics},
  volume={496},
  pages={112520},
  year={2024},
  publisher={Elsevier}
}

@article{owhadi2022computational,
  title={Computational graph completion},
  author={Owhadi, Houman},
  journal={Research in the Mathematical Sciences},
  volume={9},
  number={2},
  pages={27},
  year={2022},
  publisher={Springer}
}

@article{kinch2025structure,
  title={Structure-preserving digital twins via conditional neural whitney forms},
  author={Kinch, Brooks and Shaffer, Benjamin and Armstrong, Elizabeth and Meehan, Michael and Hewson, John and Trask, Nathaniel},
  journal={arXiv preprint arXiv:2508.06981},
  year={2025}
}

@article{trask2022enforcing,
  title={Enforcing exact physics in scientific machine learning: a data-driven exterior calculus on graphs},
  author={Trask, Nathaniel and Huang, Andy and Hu, Xiaozhe},
  journal={Journal of Computational Physics},
  volume={456},
  pages={110969},
  year={2022},
  publisher={Elsevier}
}

@article{chen2021solving,
  title={Solving and learning nonlinear PDEs with Gaussian processes},
  author={Chen, Yifan and Hosseini, Bamdad and Owhadi, Houman and Stuart, Andrew M},
  journal={Journal of Computational Physics},
  volume={447},
  pages={110668},
  year={2021},
  publisher={Elsevier}
}

@inproceedings{harkonen2023gaussian,
  title={Gaussian process priors for systems of linear partial differential equations with constant coefficients},
  author={Harkonen, Marc and Lange-Hegermann, Markus and Raita, Bogdan},
  booktitle={International conference on machine learning},
  pages={12587--12615},
  year={2023},
  organization={PMLR}
}

@article{brunton2024promising,
  title={Promising directions of machine learning for partial differential equations},
  author={Brunton, Steven L and Kutz, J Nathan},
  journal={Nature Computational Science},
  volume={4},
  number={7},
  pages={483--494},
  year={2024},
  publisher={Nature Publishing Group US New York}
}

@article{shi2025survey,
  title={A survey on machine learning approaches for uncertainty quantification of engineering systems},
  author={Shi, Yan and Wei, Pengfei and Feng, Ke and Feng, De-Cheng and Beer, Michael},
  journal={Machine learning for computational science and engineering},
  volume={1},
  number={1},
  pages={11},
  year={2025},
  publisher={Springer}
}

@article{jiao2024solving,
  title={Solving forward and inverse PDE problems on unknown manifolds via physics-informed neural operators},
  author={Jiao, Anran and Yan, Qile and Harlim, Jhn and Lu, Lu},
  journal={arXiv preprint arXiv:2407.05477},
  year={2024}
}

@article{chen2024physics,
  title={Physics-informed neural networks with hard linear equality constraints},
  author={Chen, Hao and Flores, Gonzalo E Constante and Li, Can},
  journal={Computers \& Chemical Engineering},
  volume={189},
  pages={108764},
  year={2024},
  publisher={Elsevier}
}

@article{krishnapriyan2021characterizing,
  title={Characterizing possible failure modes in physics-informed neural networks},
  author={Krishnapriyan, Aditi and Gholami, Amir and Zhe, Shandian and Kirby, Robert and Mahoney, Michael},
  journal={Advances in neural information processing systems},
  volume={34},
  pages={26548--26560},
  year={2021}
}

@article{bonfanti2024challenges,
  title={The challenges of the nonlinear regime for physics-informed neural networks},
  author={Bonfanti, Andrea and Bruno, Giuseppe and Cipriani, Cristina},
  journal={Advances in neural information processing systems},
  volume={37},
  pages={41852--41881},
  year={2024}
}

@article{arnold2010finite,
  title={Finite element exterior calculus: from Hodge theory to numerical stability},
  author={Arnold, Douglas and Falk, Richard and Winther, Ragnar},
  journal={Bulletin of the American mathematical society},
  volume={47},
  number={2},
  pages={281--354},
  year={2010}
}

@article{Arnold_Falk_Winther_2006, 
    title={Finite element exterior calculus, homological techniques, and applications}, 
    volume={15}, 
    DOI={10.1017/S0962492906210018}, 
    journal={Acta Numerica}, 
    author={Arnold, Douglas N. and Falk, Richard S. and Winther, Ragnar}, 
    year={2006}, 
    pages={1–155}
}

@article{deeponetNatureML,
author = {Lu, Lu and Jin, Pengzhan and Pang, Guofei and Zang, Handy and Karniadakis, George},
year = {2021},
month = {03},
pages = {218-229},
title = {Learning nonlinear operators via {DeepONet} based on the universal approximation theorem of operators},
volume = {3},
journal = {Nature Machine Intelligence},
doi = {10.1038/s42256-021-00302-5},
}

@article{karniadakis2021physics,
  title={Physics-informed machine learning},
  author={Karniadakis, George Em and Kevrekidis, Ioannis G and Lu, Lu and Perdikaris, Paris and Wang, Sifan and Yang, Liu},
  journal={Nature Reviews Physics},
  volume={3},
  number={6},
  pages={422--440},
  year={2021},
  publisher={Nature Publishing Group}
}

@article{yu2022gradient,
  title={Gradient-enhanced physics-informed neural networks for forward and inverse {PDE} problems},
  author={Yu, Jeremy and Lu, Lu and Meng, Xuhui and Karniadakis, George Em},
  journal={Computer Methods in Applied Mechanics and Engineering},
  volume={393},
  pages={114823},
  year={2022},
  publisher={Elsevier}
}

@article{
pideeponet,
author = {Sifan Wang  and Hanwen Wang  and Paris Perdikaris },
title = {Learning the solution operator of parametric partial differential equations with physics-informed {DeepONets}},
journal = {Science Advances},
volume = {7},
number = {40},
pages = {eabi8605},
year = {2021},
doi = {10.1126/sciadv.abi8605},
}

@article{li2020fourier,
  title={Fourier neural operator for parametric partial differential equations},
  author={Li, Zongyi and Kovachki, Nikola and Azizzadenesheli, Kamyar and Liu, Burigede and Bhattacharya, Kaushik and Stuart, Andrew and Anandkumar, Anima},
  journal={arXiv preprint arXiv:2010.08895},
  year={2020}
}

@article{xu2021accurate,
  title={Accurate remaining useful life prediction with uncertainty quantification: a deep learning and nonstationary gaussian process approach},
  author={Xu, Zhaoyi and Guo, Yanjie and Saleh, Joseph Homer},
  journal={IEEE Transactions on Reliability},
  volume={71},
  number={1},
  pages={443--456},
  year={2021},
  publisher={IEEE}
}

@article{psaros2023uncertainty,
  title={Uncertainty quantification in scientific machine learning: Methods, metrics, and comparisons},
  author={Psaros, Apostolos F and Meng, Xuhui and Zou, Zongren and Guo, Ling and Karniadakis, George Em},
  journal={Journal of Computational Physics},
  volume={477},
  pages={111902},
  year={2023},
  publisher={Elsevier}
}

@article{abdar2021review,
  title={A review of uncertainty quantification in deep learning: Techniques, applications and challenges},
  author={Abdar, Moloud and Pourpanah, Farhad and Hussain, Sadiq and Rezazadegan, Dana and Liu, Li and Ghavamzadeh, Mohammad and Fieguth, Paul and Cao, Xiaochun and Khosravi, Abbas and Acharya, U Rajendra and others},
  journal={Information fusion},
  volume={76},
  pages={243--297},
  year={2021},
  publisher={Elsevier}
}

@article{fresca2022pod,
  title={POD-DL-ROM: Enhancing deep learning-based reduced order models for nonlinear parametrized PDEs by proper orthogonal decomposition},
  author={Fresca, Stefania and Manzoni, Andrea},
  journal={Computer Methods in Applied Mechanics and Engineering},
  volume={388},
  pages={114181},
  year={2022},
  publisher={Elsevier}
}

@article{kapteyn2022data,
  title={Data-driven physics-based digital twins via a library of component-based reduced-order models},
  author={Kapteyn, Michael G and Knezevic, David J and Huynh, DBP and Tran, Minh and Willcox, Karen E},
  journal={International Journal for Numerical Methods in Engineering},
  volume={123},
  number={13},
  pages={2986--3003},
  year={2022},
  publisher={Wiley Online Library}
}

@article{jung2025accelerating,
  title={Accelerating high-fidelity simulations of chemically reacting flows using reduced-order modeling with time-dependent bases},
  author={Jung, Ki Sung and Lacey, Cristian E and Babaee, Hessam and Chen, Jacqueline H},
  journal={Computer Methods in Applied Mechanics and Engineering},
  volume={437},
  pages={117758},
  year={2025},
  publisher={Elsevier}
}

@BOOK{NAP29212,
  author    = {{National Academies of Sciences, Engineering, and Medicine} and
               {Division on Engineering and Physical Sciences} and
               {Board on Mathematical Sciences and Analytics} and
               {Committee on Foundation Models for Scientific Discovery and Innovation}},
  title     = {Foundation Models for Scientific Discovery and Innovation: Opportunities Across the Department of Energy and the Scientific Enterprise},
  isbn      = {978-0-309-99500-9},
  doi       = {10.17226/29212},
  url       = {https://nap.nationalacademies.org/catalog/29212/foundation-models-for-scientific-discovery-and-innovation-opportunities-across-the},
  year      = {2025},
  publisher = {The National Academies Press},
  address   = {Washington, DC}
}

@article{fan2023two,
  title={Two-dimensional semiconductor integrated circuits operating at gigahertz frequencies},
  author={Fan, Dongxu and Li, Weisheng and Qiu, Hao and Xu, Yifei and Gao, Si and Liu, Lei and Li, Taotao and Huang, Futao and Mao, Yun and Zhou, Wenbin and others},
  journal={Nature Electronics},
  volume={6},
  number={11},
  pages={879--887},
  year={2023},
  publisher={Nature Publishing Group UK London}
}

@InProceedings{pmlr-v107-aadithya20a,
  title = 	 {Data-driven Compact Models for Circuit  Design and Analysis},
  author =       {Aadithya, K. and Kuberry, P. and Paskaleva, B. and Bochev, P. and Leeson, K. and Mar, A. and Mei, T. and Keiter, E.},
  booktitle = 	 {Proceedings of The First Mathematical and Scientific Machine Learning Conference},
  pages = 	 {555--569},
  year = 	 {2020},
  editor = 	 {Lu, Jianfeng and Ward, Rachel},
  volume = 	 {107},
  series = 	 {Proceedings of Machine Learning Research},
  month = 	 {20--24 Jul},
  publisher =    {PMLR},
  url = 	 {https://proceedings.mlr.press/v107/aadithya20a.html},
}

@INPROCEEDINGS{9385620,
  author={Ruderman, Michael and Kaltenbacher, Stefan and Horn, Martin},
  booktitle={2021 IEEE International Conference on Mechatronics (ICM)}, 
  title={Pressure-flow dynamics with semi-stable limit cycles in hydraulic cylinder circuits}, 
  year={2021},
  volume={},
  number={},
  pages={1-6},
  doi={10.1109/ICM46511.2021.9385620}}

@book{vacca2021hydraulic,
  title={Hydraulic fluid power: Fundamentals, applications, and circuit design},
  author={Vacca, Andrea and Franzoni, Germano},
  year={2021},
  publisher={John Wiley \& Sons}
}

@article{wang2017microscale,
  title={Microscale solid-state thermal diodes enabling ambient temperature thermal circuits for energy applications},
  author={Wang, Song and Cottrill, Anton L and Kunai, Yuichiro and Toland, Aubrey R and Liu, Pingwei and Wang, Wen-Jun and Strano, Michael S},
  journal={Physical Chemistry Chemical Physics},
  volume={19},
  number={20},
  pages={13172--13181},
  year={2017},
  publisher={Royal Society of Chemistry}
}

@article{micchelli1977survey,
  title={A survey of optimal recovery},
  author={Micchelli, Charles A and Rivlin, Theodore J},
  journal={Optimal estimation in approximation theory},
  pages={1--54},
  year={1977},
  publisher={Springer}
}

@techreport{musson2022charon,
  title={Charon user manual: v. 2.2 (revision1)},
  author={Musson, Lawrence and Hennigan, Gary and Gao, Xujiao and Humphreys, Richard and Negoita, Mihai and Huang, Andy},
  year={2022},
  institution={Sandia National Laboratories (SNL-NM), Albuquerque, NM (United States)}
}

@article{song2026structure,
  title={Structure-Aware Epistemic Uncertainty Quantification for Neural Operator PDE Surrogates},
  author={Song, Haoze and Li, Zhihao and Deng, Mengyi and Li, Xin and Pan, Duyi and Lai, Zhilu and Wang, Wei},
  journal={arXiv preprint arXiv:2603.11052},
  year={2026}
}

@article{yang2019adversarial,
  title={Adversarial uncertainty quantification in physics-informed neural networks},
  author={Yang, Yibo and Perdikaris, Paris},
  journal={Journal of Computational Physics},
  volume={394},
  pages={136--152},
  year={2019},
  publisher={Elsevier}
}

@article{carlberg2018conservative,
  title={Conservative Model Reduction for Finite-Volume Models},
  author={Carlberg, Kevin and Choi, Youngsoo and Sargsyan, Syuzanna},
  journal={Journal of Computational Physics},
  volume={371},
  pages={280--314},
  year={2018},
  doi={10.1016/j.jcp.2018.05.019},
  publisher={Elsevier}
}

@article{benner2015survey,
  title={A Survey of Projection-Based Model Reduction Methods for Parametric Dynamical Systems},
  author={Benner, Peter and Gugercin, Serkan and Willcox, Karen},
  journal={SIAM Review},
  volume={57},
  number={4},
  pages={483--531},
  year={2015},
  doi={10.1137/130932715},
  publisher={SIAM}
}

@book{quarteroni2015reduced,
  title={Reduced Basis Methods for Partial Differential Equations: An Introduction},
  author={Quarteroni, Alfio and Manzoni, Andrea and Negri, Federico},
  series={Unitext},
  volume={92},
  publisher={Springer},
  address={Cham},
  year={2016},
  doi={10.1007/978-3-319-15431-2}
}

@conference{osti_1575982,
  author       = {Musson, Lawrence and Gao, Xujiao and Negoita, Mihai and Huang, Andy and Hennigan, Gary L.},
  title        = {Charon: A Radiation Aware Massively Parallel TCAD Modeling Code.},
  url          = {https://www.osti.gov/biblio/1575982},
  place        = {United States},
  organization = {Sandia National Laboratories (SNL-NM), Albuquerque, NM (United States)},
  year         = {2018},
  month        = {11}}

\newpage

\appendix
\section{Proof for \Cref{prop:rank}}\label{app:proof_prop_rank}
\begin{proof}For a boundary edge $e\in\mathcal E_\gamma$, the divergence $\nabla\!\cdot\boldsymbol\phi_e^1$ is supported entirely on its unique adjacent cell $K(e)$ and is zero everywhere else. Let $\beta_e^\gamma=(\nabla\!\cdot\boldsymbol\phi_e^1,\phi_{K(e)}^{P_0})_\Omega$. Since the coarse 0-form and boundary 1-form are defined as 
\begin{equation*}
    \psi_i^0=\sum_{a\in\mathcal C_h} W_{ia}\phi_a^{P_0},\qquad \boldsymbol\psi_{\alpha,\gamma}^{1,\mathrm{bc}}
=
\sum_{e\in\mathcal E_\gamma}W_{\alpha,K(e)}\boldsymbol\phi_e^1,
\end{equation*}

we have
\begin{align*}
    (D_\gamma)_{i\alpha}&=
(\nabla\!\cdot\boldsymbol\psi_{\alpha,\gamma}^{1,\mathrm{bc}},\psi_i^0)_\Omega \\
&=
\sum_{e\in\mathcal E_\gamma}
W_{\alpha,K(e)}
(\nabla\!\cdot\boldsymbol\phi_e^1,\psi_i^0)_\Omega \\
&=\sum_{e\in\mathcal E_\gamma}
W_{\alpha,K(e)}W_{i,K(e)}
(\nabla\!\cdot\boldsymbol\phi_e^1,\phi_{K(e)}^{P_0})_\Omega\\
&=\sum_{e\in\mathcal E_\gamma}
\beta_e^\gamma W_{\alpha,K(e)}W_{i,K(e)}
\end{align*}
This shows $D_\gamma=W_\gamma B_\gamma W_\gamma^\top$, where $(W_\gamma)_{i e}=W_{i,K(e)},\ B_\gamma=\operatorname{diag}\{\beta_e^\gamma:e\in\mathcal E_\gamma\}$.

Then it remains to check the sufficient condition. Let $$D_\partial=[D_1\;D_2\;\cdots\;D_r]$$ denote the boundary block of the divergence matrix $D$, collecting the boundary groups $\gamma=1,\dots,r$. Given $D\in\mathbb{R}^{N_0\times N_1}$, we have $$\operatorname{rank}(D)\le N_0.$$
If $D_\partial$ has full row rank $N_0$, since $D_\partial$ is the column submatrix of $D$ corresponding to the boundary basis functions, we have
$$\operatorname{rank}(D)
\ge\operatorname{rank}(D_\partial)=N_0,$$ and therefore $\operatorname{rank}(D)=N_0$. 

Let $G_\partial=\sum_{\gamma=1}^r W_\gamma B_\gamma W_\gamma^\top$ be positive definite. For any $x\in\mathbb R^{N_0}$ such that $x^\top D_\partial=0$, we have 
$$0=x^\top\left(\sum_{\gamma=1}^rD_\gamma\right)x=x^\top G_\partial x \Longrightarrow x=0$$
since $G_\partial\succ0$. Hence \(\operatorname{Null}(D_\partial^\top)=\{0\}\), and by the rank--nullity theorem, \(D_\partial\) has full row rank \(N_0\).
\end{proof}

\end{document}